\documentclass{article}
\usepackage[group-separator={,}]{siunitx}

\usepackage{slashbox}
\usepackage{booktabs} % For formal tables
\usepackage{color}
\usepackage{listings}
\usepackage{enumitem}
\usepackage{threeparttable}
\usepackage{multirow}% http://ctan.org/pkg/multirow
\usepackage{hhline}% http://ctan.org/pkg/hhline
\usepackage{subfigure}
\usepackage{array}
\usepackage{pifont}
\usepackage{booktabs}
\usepackage{dsfont}
\usepackage{microtype}
\usepackage{graphicx}
\usepackage{stmaryrd}
\usepackage{amsmath}
\usepackage{amsthm}
\usepackage{amssymb}
\usepackage{url}
\usepackage{xspace}
\usepackage[normalem]{ulem}
\usepackage{booktabs, siunitx}
\usepackage[svgnames,table]{xcolor}
\usepackage[tableposition=above]{caption}
\usepackage{xcolor}
\usepackage{tkz-euclide}
\usepackage[framemethod=tikz]{mdframed}

\newtheorem{theorem}{Theorem}

\newmdenv[innerlinewidth=0.5pt, roundcorner=4pt,innerleftmargin=6pt,
innerrightmargin=6pt,innertopmargin=6pt,innerbottommargin=6pt]{mybox}

\expandafter\def\expandafter\UrlBreaks\expandafter{\UrlBreaks%  save the current one
  \do\a\do\b\do\c\do\d\do\e\do\f\do\g\do\h\do\i\do\j%
  \do\k\do\l\do\m\do\n\do\o\do\p\do\q\do\r\do\s\do\t%
  \do\u\do\v\do\w\do\x\do\y\do\z\do\A\do\B\do\C\do\D%
  \do\E\do\F\do\G\do\H\do\I\do\J\do\K\do\L\do\M\do\N%
  \do\O\do\P\do\Q\do\R\do\S\do\T\do\U\do\V\do\W\do\X%
  \do\Y\do\Z}
  
\newcommand{\oracle}[0]{exhaustive}
\newcommand{\Oracle}[0]{Exhaustive}
\newcommand{\Oracleshort}[0]{Exhaustive}

\newcommand{\nnmodel}[0]{F}
\newcommand{\params}[0]{\theta}

\newcommand{\perturbe}[0]{R}

\newcommand{\Tinsword}{T_{\emph{Dup}}}
\newcommand{\Tdelword}{T_{\emph{DelStop}}}
\newcommand{\Tsubword}{T_{\emph{SubSyn}}}
\newcommand{\Tsubchar}{T_{\emph{SubAdj}}}
\newcommand{\Tswapchar}{T_{\emph{SwapPair}}}
\newcommand{\Tinschar}{T_{\emph{InsAdj}}}
\newcommand{\Tdelchar}{T_{\emph{Del}}}

\newcommand{\aat}{A3T\xspace}

\DeclareMathOperator{\bfx}{{\bf x}}
\DeclareMathOperator{\bfy}{{\bf y}}
\DeclareMathOperator{\bfz}{{\bf z}}
\DeclareMathOperator{\bfs}{{\bf s}}
\DeclareMathOperator{\bfv}{{\bf v}}

\DeclareMathOperator*{\argmin}{{\rm argmin}}
\DeclareMathOperator{\calD}{{\mathcal D}}
\newcommand{\dataset}{D}
\DeclareMathOperator{\calX}{{\mathcal X}}
\DeclareMathOperator{\calY}{{\mathcal Y}}

\DeclareMathOperator{\alphabet}{\Sigma}
\newcommand{\spec}{S}
\newcommand{\that}{\widehat{T}}

\theoremstyle{definition}
\newtheorem{definition}{Definition}
\newtheorem{example}{Example}
\newtheorem{lemma}[theorem]{Lemma}

\newcommand{\advab}{\aat{}(HotFlip)}
\newcommand{\enumab}{\aat{}(search)}
\newcommand{\lengththreecell}{4.5cm}
\newcommand{\lengthcell}{1.05cm}
\newcommand{\lengthcellshorter}{0.8cm}
\newcommand{\lengthcelllonger}{1.75cm}

\renewcommand{\leq}{\leqslant}

\newcommand{\perturbedsample}{perturbed sample}

\newcommand{\perturbed}[1]{{\color{red}{#1}}}

\newcommand{\yhmodify}[2]{{#2}}
\usepackage{hyperref}
\usepackage[capitalize]{cleveref}

\usepackage[accepted]{icml2020}

\icmltitlerunning{Robustness to String Transformations via Augmented Abstract Training}

\begin{document}

\twocolumn[
\icmltitle{Robustness to Programmable String Transformations \\ via Augmented Abstract Training}



\icmlsetsymbol{equal}{*}

\begin{icmlauthorlist}

\icmlauthor{Yuhao Zhang}{csuwm}
\icmlauthor{Aws Albarghouthi}{csuwm}
\icmlauthor{Loris D'Antoni}{csuwm}
\end{icmlauthorlist}

\icmlaffiliation{csuwm}{Department of Computer Science, University of Wisconsin-Madison, Madison, USA}

\icmlcorrespondingauthor{Yuhao Zhang}{yuhaoz@cs.wisc.edu}

\icmlkeywords{Machine Learning, ICML}

\vskip 0.3in
]



\printAffiliationsAndNotice{}  % leave blank if no need to mention equal contribution

\begin{abstract}
Deep neural networks for natural language processing tasks are vulnerable to adversarial input perturbations.
In this paper, we present a versatile language for programmatically specifying string transformations---e.g., insertions, deletions, substitutions, swaps, etc.---that are relevant to the task at hand.
We then present an approach to adversarially training models that are robust to such user-defined string transformations. 
Our approach combines the advantages of search-based techniques for adversarial training with abstraction-based techniques.
Specifically, we show how to decompose a set of user-defined string transformations into two component specifications, one that benefits from search and another from abstraction. 
We use our technique to train models on the AG and SST2 datasets and show that the resulting models are robust to  combinations of user-defined transformations mimicking spelling mistakes and other meaning-preserving transformations.
\end{abstract}

\section{Introduction} \label{sec:intro}
Deep neural networks have proven incredibly powerful in a huge range of machine-learning tasks.
However, deep neural networks are highly sensitive to small input perturbations that cause the network's accuracy to plummet~\cite{DBLP:conf/ccs/Carlini017,szegedy2013intriguing}.
In the context of natural language processing, these \emph{adversarial examples} come in the form of spelling mistakes, use of synonyms, etc.---essentially, meaning-preserving transformations that cause the network to change its prediction~\cite{DBLP:conf/acl/EbrahimiRLD18, DBLP:conf/acl/ZhangZML19, DBLP:journals/corr/abs-1903-06620}.
% In particular, in Natural Language Processing (NLP), 
% models are vulnerable  to
% small character and word perturbations~\cite{DBLP:conf/acl/EbrahimiRLD18, DBLP:conf/acl/ZhangZML19, DBLP:journals/corr/abs-1903-06620}---e.g., a small typo in a sentence can change the model prediction. 
% These vulnerabilities have sparked a great deal of research on
% techniques for training models that are robust to adversarial
% input perturbations.

In this paper, we are interested in the problem of training models over natural language---or, generally, sequences over a finite alphabet---that are \emph{robust} to adversarial examples.
Sequences over finite alphabets are unique in that the space of adversarial examples is discrete and therefore hard to explore efficiently using gradient-based optimization as in the computer-vision setting.
The common approach to achieving robustness is \emph{adversarial training}~\cite{DBLP:journals/corr/GoodfellowSS14,DBLP:conf/iclr/MadryMSTV18}, which has seen a great deal of research in computer vision and, more recently, in natural language processing~\cite{DBLP:conf/acl/EbrahimiRLD18, DBLP:conf/acl/ZhangZML19, DBLP:journals/corr/abs-1903-06620}.
Suppose we have defined a space of perturbations $\perturbe{}(x)$ of a sample $x$---e.g.,
if $x$ is a sentence, $\perturbe{}(x)$ contains every possible misspelling of words in $x$, up to some bound on the number of misspellings.
The idea of adversarial training is to model an adversary within the training objective function:
Instead of computing the loss for a sample $(x,y)$ from the dataset, we compute the loss for the worst-case \perturbedsample{} $z \in \perturbe{}(x)$.
Formally, the adversarial loss for $(x,y)$ is $\max_{z\in \perturbe{}(x)} \mathcal{L}(z,y,\theta)$.

The question we ask in this paper is:
\begin{center}
\emph{Can we train models that are robust against rich perturbation spaces over strings?}
\end{center}
The practical challenge in answering this question is computing the worst-case loss.
This is because the perturbation space $\perturbe{}(x)$ can be enormous and therefore impractical to enumerate.
This is particularly true for NLP tasks, where the perturbation space $\perturbe{}(x)$ should contain inputs that are semantically equivalent to $x$---e.g., variations of the sentence $x$ with typos or words replaced by synonyms.
Therefore, we need to \emph{approximate} the adversarial loss.
There are two such classes of approximation techniques:
\begin{description}
\item[Augmentation] The first class of techniques computes a \emph{lower bound} on the adversarial loss by exploring a finite number of points in $\perturbe{}(x)$. This is usually done by applying a gradient-based attack, like HotFlip~\cite{DBLP:conf/acl/EbrahimiRLD18} for natural-language tasks or PGD~\cite{DBLP:conf/iclr/MadryMSTV18} for computer-vision tasks.
We call this class of techniques \emph{augmentation}-based, as they essentially search for a \perturbedsample{} with which to augment the training set.
\item[Abstraction] The second class of techniques computes an \emph{upper bound} on the adversarial loss by overapproximating, or \emph{abstracting}, the perturbation space $\perturbe{}(x)$ into a set of symbolic constraints that can be efficiently propagated through the network.
For example, the \emph{interval} abstraction has been used in numerous works~\cite{diffai, DBLP:conf/iccv/GowalDSBQUAMK19,DBLP:conf/emnlp/HuangSWDYGDK19}.
We call this class of techniques \emph{abstraction}-based.
\end{description}

Both classes of techniques can produce suboptimal results:
augmentation can severely underapproximate the worst-case loss and abstraction can severely overapproximate the loss.
Particularly, we observe that the two techniques have complementary utility, working well on some perturbation spaces but not others---for example, \citet{DBLP:conf/emnlp/HuangSWDYGDK19}
have shown that abstraction works better for token substitutions,
while augmentation-based techniques like HotFlip~\cite{DBLP:conf/acl/EbrahimiRLD18} and MHA~\cite{DBLP:conf/acl/ZhangZML19} are general---e.g., apply to token deletions and insertions.

\begin{figure}[t]
    \centering
    \includegraphics[scale=1.25]{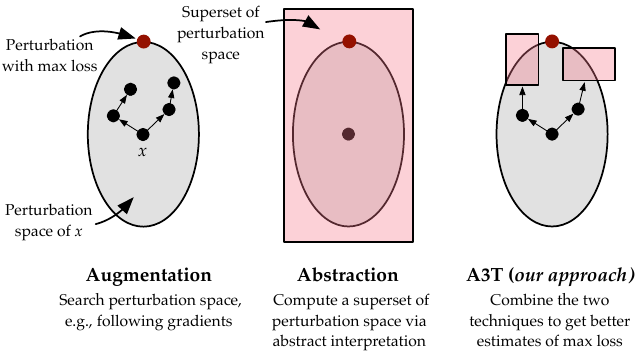}
    \caption{Illustration of augmentation, abstraction, and \aat}
    \label{fig:overview}
\end{figure}

\subsection{Our Approach}
\textbf{A hybrid approach}
We propose \emph{augmented abstract adversarial training} (\aat), an adversarial training technique that combines the strengths of augmentation and abstraction techniques.
The key idea underlying \aat is to decompose the perturbation space into two subsets, one that can be explored using augmentation and one that can be abstracted---e.g., using augmentation to explore word duplication typos and abstraction to explore replacing words with synonyms.
From an algorithmic perspective, our computation of adversarial loss switches from a concrete, e.g., gradient-based, search through the perturbation space to a symbolic search.
As such, for every training sample $(x,y)$,
our technique may end up with a lower bound \emph{or} an upper bound on its adversarial loss (see \cref{fig:overview}).

\textbf{A language for specifying string transformations}
The challenge of applying \aat is how to exactly decompose the perturbation space.
Our key enabling idea is to define the perturbation space \emph{programmatically},
in a way that can be easily and cleanly decomposed.
% \{the following para is a mess. it mixes T instead of S with $T_1...T_n$. Yuhao please rewrite, I'll read once you are done.}
Specifically, we define a general language in which we can specify a perturbation space by a specification $\spec$ in the form of $\{(T_1,\delta_1), \ldots,(T_n, \delta_n)\}$, containing a set of string transformations $T_i:\mathcal{X} \to 2^\mathcal{X}$.
The specification $\spec{}$ defines a perturbation space of all possible strings by applying each transformation $T_i$ up to $\delta_i$ times.
For example, given a string $x$, $S(x)$ could define the set of all strings $x'$ that are like $x$ but with some words replaced by one of its synonyms and with some stop words removed.

Given a perturbation space defined by a set of transformations, \aat decomposes the set of transformations into two disjoint subsets, one that is explored concretely (augmentation) and one that is explored symbolically (abstraction).

\textbf{Results}
We have implemented \aat and used it to train NLP models for sentiment analysis
that are robust to a range of string transformations---e.g., 
character swaps modeling spelling mistakes,
substituting of a word with a synonym, removing stop words, duplicating words, etc. 
Our results show that \aat can train models that are more robust to adversarial string transformations than those produced using existing techniques.

\subsection{Summary of Contributions}
\begin{itemize}
    \item We present \aat, a technique for training models that are robust to string transformations. \aat combines search-based attacks and abstraction-based techniques to explore the perturbation space and
    compute good approximations of adversarial loss. 
    \item To enable \aat, we define a general language of string transformations with which we can specify the perturbation space. \aat exploits the specification to decompose and search the perturbation space.
    \item We implement \aat\footnote{We provide our code at \url{https://github.com/ForeverZyh/A3T}.} and evaluate it on two datasets and a variety of string transformations. Our results demonstrate the increase in robustness achieved by \aat in comparison with state-of-the-art techniques.
\end{itemize}
%https://anonymous.4open.science/r/5ce3e069-b9be-4105-b8ae-2798f8dc8575/

\section{Related Work}
\textbf{Adversarial text generation}
\citet{zhang2019adversarial} presented a comprehensive overview of adversarial attacks on neural networks over natural language.
In this paper, we focus on the word- and character-level.
HotFlip~\cite{DBLP:conf/acl/EbrahimiRLD18} is a gradient-based approach that can generate the adversarial text in the perturbation space described by word- and character-level transformations. 
 MHA~\cite{DBLP:conf/acl/ZhangZML19} uses Metropolis-Hastings sampling guided by gradients to generate word-level adversarial text via word substitution.
\citet{DBLP:conf/aclnut/KarpukhinLEG19} designed character-level noise for training robust machine translation models.
Other work focuses on generating adversarial text on the sentence-level~\cite{DBLP:conf/ijcai/0002LSBLS18} or paraphrase-level~\cite{DBLP:conf/naacl/IyyerWGZ18,DBLP:conf/acl/SinghGR18}.
Also, some works~\cite{DBLP:conf/iclr/ZhaoDS18,DBLP:journals/corr/abs-1912-10375} try to generate natural or grammatically correct adversarial text.
% Another kind of adversary is universal adversarial triggers~\cite{DBLP:conf/emnlp/WallaceFKGS19} that are fixed words which concatenate to any input will trigger the model to make a false prediction.

\textbf{Abstract training}
\citet{diffai}
and \citet{DBLP:conf/iccv/GowalDSBQUAMK19} first proposed \textit{DiffAI} and interval bound propagation (IBP) to train image classification models that are provably robust to norm-bounded adversarial perturbations.
They performed abstract training by optimizing the abstract loss obtained by Interval or Zonotope propagation.
\yhmodify{}{
\citet{DBLP:journals/corr/abs-1903-12519} used DiffAI to defend deep residual networks provably.
\citet{DBLP:conf/emnlp/JiaRGL19} used interval domain to capture the perturbation space of substitution and train robust models for CNN and LSTM.
}
\citet{DBLP:conf/emnlp/HuangSWDYGDK19} proposed a simplex space to capture the perturbation space of substitution.
They converted the simplex into intervals after the first layer of the neural network and obtained the abstract loss by IBP.
We adopt their abstract training approach for some of our transformations.
We will show the limitation of abstract training for more complex perturbations like the combination of swap and substitution.

\textbf{Other robustness techniques}
Other techniques to ensure robustness involve placing a spelling-mistake-detection model that identifies possible adversaries before the underlying model \cite{DBLP:conf/acl/PruthiDL19, DBLP:conf/aaai/SakaguchiDPD17}.

\textbf{Formal verification for neural networks}
In the verification for NLP tasks,
\citet{shi2020robustness} combined forward propagation and a tighter backward bounding process to achieve the formal verification of Transformers.
\citet{welbl2020towards} proposed the formal verification under text deletion for models based on the popular decomposable attention mechanism by interval bound propagation.
\textit{POPQORN}~\cite{DBLP:conf/icml/KoLWDWL19} is a general algorithm to quantify the robustness of recurrent neural networks, including RNNs, LSTMs, and GRUs.
\yhmodify{}{
COLT~\cite{DBLP:conf/iclr/BalunovicV20} combines formal verification and adversarial training to train neural networks.
% \citet{DBLP:conf/nips/SalmanLRZZBY19} proposed adversarial training to boost the provable robustness of randomized smoothing.
}
% \citet{DBLP:conf/emnlp/HuangSWDYGDK19} proposed a simplex space to capture the perturbation space of substitution and used IBP to verify the neural network.
In this paper, we mix verification techniques, namely, interval propagation, with search-based techniques.
% In the field of verification for image classification, AI2~\cite{DBLP:conf/sp/GehrMDTCV18}, DeepZ~\cite{DBLP:conf/nips/SinghGMPV18}, DeepPoly~\cite{DBLP:journals/pacmpl/SinghGPV19}, and RefineZono~\cite{DBLP:conf/iclr/SinghGPV19} are tools based on abstract interpretation for robustness verification of neural networks.

\section{The Perturbation-Robustness Problem}

In this section, we (1) formalize the perturbation-robustness problem
and (2) define a string transformation language for specifying the perturbation space.

\subsection{Perturbation Robustness}
\textbf{Classification setting}
We consider a standard classification setting with samples from some domain $\calX$ and labels from $\calY$.
Given a distribution $\calD$ over samples and labels, our goal is to find the optimal parameters
$\params$ of some neural-network architecture $\nnmodel_\params$ that minimize the expected loss
\begin{align}
    \argmin_{\params} \mathop{\mathbb{E}}_{(\bfx,y)\sim \mathcal{D}} \mathcal{L}(\bfx,y,\params) \label{eq:normal}
\end{align}

We are interested in the setting where the sample space $\calX$ defines strings over some finite \emph{alphabet} $\alphabet$. The alphabet $\alphabet$ can be, for example, English characters (in a character-level model) or entire words (in a world-level model).
Therefore, the domain $\calX$ in our setting is $\alphabet^*$, i.e., the set of all strings of elements of $\alphabet$.
We will use $\bfx \in \alphabet^*$ to denote a string and $x_i \in \alphabet$ to denote the $i$th element of the string.

% Solving the classification problem in NLP can be seen as finding a machine learning algorithm $H: \mathcal{X} \mapsto \mathcal{Y}$ that takes a string $\singlestring{}$ and predicts its label $y$. 
% If $H$ are neural networks, we will try to find a set of parameters $\nnmodel{}$ that minimize the empirical risk in a distribution over $\mathcal{X} \times \mathcal{Y}$:
% $$\min_{\nnmodel{}} \mathop{\mathbb{E}}_{(x,y)\sim \mathcal{D}} \mathcal{L}(x,y,\nnmodel{})$$
% , where $\mathcal{L}$ denotes the loss function.
% Since the data distribution $\mathcal{D}$ is unknown, we will use a dataset $D=\{(x_1,y_1),\dots, (x_n,y_n)\}$ to approximate the distribution $\mathcal{D}$:
% $$\min_{\nnmodel{}} \frac{1}{n}\sum_{i=1}^n \mathcal{L}(x_i,y_i,\nnmodel{})$$

\textbf{Perturbation space}
We define a \emph{perturbation space} $\perturbe$ as a function in $\alphabet^* \to 2^{\alphabet^*}$, i.e., $\perturbe$ takes a string $\bfx$ and returns a finite set of possible \emph{perturbed} strings of $\bfx$.
% The perturbation space $\perturbe{}(x)$ is the outputs by applying $\perturbe{}$ to the input $x$.
%

We will use a perturbation space to denote a set of strings that should receive the same classification by our network.
For example, $\perturbe(\bfx)$ could define a set of sentences paraphrasing $\bfx$.
We can thus modify our training objective into a \emph{robust-optimization} problem, following
\citet{DBLP:conf/iclr/MadryMSTV18}:
\begin{align}\argmin_{\params} \mathop{\mathbb{E}}_{(\bfx,y)\sim \mathcal{D}} \; \max_{\bfz\in\perturbe(\bfx)} \; \mathcal{L}(\bfz,y,\params) \label{eq:advtrain}
\end{align}
This inner objective is usually hard to solve; in our setting, the perturbation space can be very large and we cannot afford to consider every single point in that space during training. 
Therefore, as we discussed in \cref{sec:intro}, typically approximations are made.

\textbf{\Oracle{} accuracy}
Once we have trained a model $\nnmodel_\params$ using the robust optimization objective, we will use \emph{\oracle{} accuracy} to quantify its classification accuracy in the face of perturbations.
Specifically, given a dataset $\dataset = \{(\bfx_i,y_i)\}_{i=1}^n$
and a perturbation space $\perturbe$, we define \oracle{} accuracy as follows:
\begin{align}
    \frac{1}{n} \sum_{i=1}^n \mathds{1}[\forall \bfz \in \perturbe(\bfx_i).\ \nnmodel_\params(\bfz)=y_i] \label{eq: exhaustiveacc}
\end{align}
Intuitively, for each sample $(\bfx_i,y_i)$, its classification is considered correct iff $\nnmodel_\params$ predicts $y_i$ for every single point in $\perturbe(\bfx_i)$.
We use \oracle{} accuracy instead of the commonly used adversarial accuracy because 
(1) \oracle{} accuracy provides the ground truth accuracy of the discrete perturbation spaces and does not depend on an underlying adversarial attack,
and (2) the discrete spaces make it easy for us to compute \oracle{} accuracy by enumeration and at the same time hard for the gradient-based adversarial attacks to explore the space.

\subsection{A Language for Specifying Perturbations}
We have thus far assumed that the perturbation space is provided to us. 
We now describe a language for modularly specifying a perturbation space.

% In this paper, we let the users define the perturbations $\perturbe{}$ to mimic spelling mistakes and other meaning-preserving transformations. 
% Such perturbations $\perturbe{}$ can be defined in the following language:
% \begin{align*}
%     \perturbe{} := &(T,\delta) \mid  \perturbe{} \circ \perturbe{}
% \end{align*}
% And $F$ can be seen as a set of string transformations $T_i$ with their perturbation sizes $\delta_i$, which restricts how many times $T_i$ can be applied to the original input at most.

A specification  $\spec$ is defined as follows:
$$\spec=\{(T_1, \delta_1), \dots ,(T_n, \delta_n)\}$$
where each $T_i$ denotes a string transformation that can be applied
up to $\delta_i \in \mathbb{N}$ times.
Formally, a string transformation $T$ is a pair $(\varphi, f)$,
where 
$\varphi: \alphabet^* \to \{0,1\}$ is a \textit{Boolean predicate} (in practice, a regular expression) describing the substrings of the inputs to which the transformation can be applied, and
$f: \alphabet^* \to 2^{\alphabet^*}$ is a \emph{transformer} describing
how the substrings matched by $\varphi$ can be replaced.

\textbf{Single transformations}
% Given a transformation $T=(\varphi, f)$ and a string 
% $\bfx$, we define the perturbation space $T(\bfx)$ as follows.
% Let $\{(l_i,r_i)\}_{i\leq k}$ be the set of 
% starting and ending positions of all substrings
% of $\bfx$ that match $\varphi$.
% $T(\bfx)$ is then defined as the set:
% \[
% \begin{array}{l}
%      \{x_1\ldots x_{l_i-1}\bfs_ix_{r_i+1}\ldots x_n  \mid  
%      \bfs_i\in f(x_{l_i}\ldots x_{r_i}) \}
% \end{array}
% \]
% Informally, each string $\bfz\in T(\bfx)$ is
% the result of replacing one substring  matched by $\varphi$ in $\bfx$ with one of the possible results obtained by applying $f$
% to that substring.
Before defining the semantics of our specification language, we illustrate a few example specifications involving single transformations:

% \{proposed to remove this example to shrink the space. I've done the liveness check and it's safe to move.}
\begin{example}[$T_\emph{stop}=(\varphi_\emph{stop}, f_\emph{stop})$]
Suppose we want to define a transformation that deletes a stop word---\emph{and}, \emph{the}, \emph{is}, etc.---mimicking a typo.
The predicate $\varphi_\emph{stop}$ will be a regular expression matching all stop words.
The transformer $f_\emph{stop}$ will be simply the function that takes a string and returns the set containing the empty string, $f_\emph{stop}(\bfx) = \{\epsilon\}$.
Consider a specification $\spec_\emph{stop}=\{(T_\emph{stop},1)\}$ that applies
the transformation $T_\emph{stop}$ up to one time.
On the following string,
    \emph{They are at school},
the predicate $\varphi_\emph{stop}$ matches the substrings \emph{are} and \emph{at}. In both cases, we apply the predicate $f_\emph{stop}$ to the matched word and insert the output of $f_\emph{stop}$ in its position.
This results in the set containing
the original string (0 transformations are applied) and the two strings
    \emph{They at school} and
    \textit{They are school}.
Applying a specification $\spec_\emph{stop}^2=\{(T_\emph{stop},2)\}$, which is allowed to apply $T_\emph{stop}$ at most twice, to the same input would result in a set of strings containing the strings above as well as the string \emph{They school}.
\end{example}

\begin{example}[$T_\emph{nice}=(\varphi_\emph{nice}, f_\emph{nice})$]
% \{see how I added $\spec$ earlier and do same}
Say we want to transform occurrences of \emph{nice} into one of its synonyms, \emph{enjoyable} and \emph{pleasant}.
We define the predicate $\varphi_\emph{nice}(\bfx)$ that is true iff $\bfx = \emph{nice}$,
and we define $f_\emph{nice}(\bfx) = \{\emph{enjoyable}, \emph{pleasant}\}$.
Given the string \emph{This is nice!}, it will be transformed into the set \{\emph{This is enjoyable!}, \emph{This is pleasant!}\}.

Applying a specification $\spec_\emph{nice}^2=\{(T_\emph{nice},2)\}$ to the same input would result in the same set of strings above.
Because the predicate $\varphi_\emph{nice}(\bfx)$ only matches the word $\emph{nice}$.
\end{example}

\begin{example}[$T_\emph{swap}=(\varphi_\emph{swap}, f_\emph{swap})$]
Now consider the case where we would like to swap adjacent vowels.
$\varphi_\emph{swap}$ will be defined as the regular expression that matches two adjacent characters that are vowels.
Next, since $\bfx=x_0x_1$ can only have length 2, the transformer $f_\emph{swap}$ will be the swap function $f_\emph{swap}(x_0x_1) = \{x_1x_0\}$.
\end{example}

\textbf{Multiple transformations}
As discussed above, a specification $\spec$ in our language is a set of transformations 
$\{(T_1, \delta_1), \dots , (T_n, \delta_n)\}$
where each $T_i$ is a pair $(\varphi_i, f_i)$.
The formal semantics of our language can be found in the supplementary Appendix.
% Semantically, we consider every set of transformation applications
% that has 0 to $\delta_i$ occurrences of transformation $T_i$,
% and generate all possible strings by applying transformations.
% 
Informally, 
a string $\bfz$ is in the perturbation space $\spec(\bfx)$
if it can be obtained by
(1) finding a set $\sigma$ of \textit{non-overlapping} substrings of $\bfx$ that match the various predicates $\varphi_i$ and such that at most $\delta_i$ substrings in $\sigma$ are matches of $\varphi_i$,
and (2) replacing each substring $\bfx'\in\sigma$ matched by $\varphi_i$ with  a string in $f_i(\bfx')$.
The complexity of the formalization is due to the requirement that
matched substrings should not overlap---this requirement guarantees that each character in the input is only involved in a single transformation and will be useful when formalizing our abstract training approach in \cref{ssec:abstract}.

% We illustrate the idea with an example.

\begin{example}[Multiple Transformations]
Using the transformations $T_\emph{nice}$ and $T_\emph{swap}$
we can define  the specification 
\[\spec_\emph{ns}=\{(T_\emph{nice}, 1), (T_\emph{swap}, 1)\}\]
Then, $\spec_\emph{ns}(\emph{This house is nice})$ results in the set of strings:

\begin{tabular}{>{\em}c >{\em}c }
This house is nice &
This house is \textbf{enjoyable} \\
This house is \textbf{pleasant} &
This h\textbf{uo}se is nice \\
This h\textbf{uo}se is \textbf{enjoyable} &
This h\textbf{uo}se is \textbf{pleasant}
\end{tabular}

The transformed portions are shown in bold.
Note that we apply \textit{up to}  $1$ of each transformation,
thus we also get the original string.
Also, note that the two transformations cannot modify 
overlapping substrings of the input; for example, $T_\emph{swap}$ did not swap the \emph{ea} in \emph{pleasant}.
\end{example}

\section{Augmented Abstract Adversarial Training}

In this section, we describe our abstract training technique, \aat, which combines augmentation and abstraction.

Recall the adversarial training objective function, \cref{eq:advtrain}.
The difficulty in solving this objective is the inner maximization objective: 
$\max_{\bfz\in\perturbe(\bfx)} \; \mathcal{L}(\bfz,y,\params),
$
where the perturbation space $\perturbe(\bfx)$ can be intractably large to efficiently enumerate, and we therefore have to resort to approximation.
We begin by describing two approximation techniques and then discuss how our approach combines and extends them.

\textbf{Augmentation (search-based) techniques}
We call the first class of techniques \emph{augmentation} techniques, since they search for a worst-case sample in the perturbation space $\perturbe(\bfx)$ with which to augment the dataset.
The na\"ive way is to simply enumerate all points in $\perturbe(\bfx)$---%
our specifications induce a finite perturbation space, by construction.
Unfortunately, this can drastically slow down the training.
For example, suppose $T$ defines a transformation that swaps two adjacent characters.
On a string of length $N$, the specification $(T,2)$ results in $O(N^2)$ transformations.
% \todo{check math}.
% \loris{I think the about captures that, if we write $O(N^2)$ it is true}
% \yh{checked}

An efficient alternative, HotFlip, was proposed by \citet{DBLP:conf/acl/EbrahimiRLD18}.
HotFlip efficiently encodes a transformation $T$ as an operation over the embedding vector and \emph{approximates} the worst-case loss using a single forward and backward pass through the network.
To search through a set of transformations, HotFlip employs a beam search of some size $k$ to get the top-$k$ \perturbedsample{s}.
This technique yields a point in $\perturbe(\bfx)$ that may not have the worst-case loss.
% \yh{add the following}
Alternatives like MHA~\cite{DBLP:conf/acl/ZhangZML19} can also be used as augmentation techniques.
% \aws{we should say there are alternatives}

\textbf{Abstraction techniques}
Abstraction techniques compute an over-approximation of the perturbation space,
as a symbolic set of constraints.
This set of constraints is then propagated through the network, resulting in an upper bound on the worst-case loss.
Specifically, given a transformation $T$, we define a corresponding abstract transformation $\that$
such that for all $\bfx$, the 
constraint $T(\bfx) \subseteq \that(\bfx)$ holds.

Our use of abstraction builds upon the work of \citet{DBLP:conf/emnlp/HuangSWDYGDK19},
which uses an \emph{interval domain} to define $\that(\bfx)$
\yhmodify{---i.e., 
$\that(\bfx)$ is a conjunction of constraints of the form $l_i{\leq} x_i{\leq} u_i$, where $l_i$ and $u_i$ are constants.}{---i.e., $\that(\bfx)$ is a conjunction of constraints on each character.}
We will describe how we generalize their approach in \cref{ssec:abstract};
for now, we assume that we can efficiently overapproximate the worst-case loss for $\that(\bfx)$
by propagating it through the network.

\subsection{\aat: A High-Level View}
The key idea of \aat is to decompose a specification $\spec$ into two sets of transformations,
one containing transformations that can be effectively explored with augmentation and 
one containing transformations that can be precisely abstracted.

\cref{alg:a3t} shows how  
\aat works.
First, we decompose the specification $\spec$ into two subsets of transformations, resulting in two specifications, $\spec_\emph{aug}$ and $\spec_\emph{abs}$.
For $\spec_\emph{aug}$, we apply an augmentation technique, e.g., HotFlip or MHA, to come up with a list of top-$k$ \perturbedsample{s} in the set $\spec_\emph{aug}(\bfx)$---this is denoted as the set $\emph{augment}_k(\spec_\emph{aug},\bfx)$.

Then, for each point $\bfz$ in the top-$k$ results, we compute an abstraction
$\mathit{abstract}(\spec_\emph{abs},\bfz)$, which is a set of constraints over-approximating the set of points in $\spec_\emph{abs}(\bfz)$.
Recall our overview in \cref{fig:overview} for a visual depiction of this process.

Finally, we return the worst-case loss.

\begin{algorithm}[tb]
   \caption{A3T}
   \label{alg:a3t}
\begin{algorithmic}
   \STATE {\bfseries Input:} $\spec=\{(T_1,\delta_1), \dots , (T_n,\delta_n)\}$ and point $(\bfx,y)$
   \STATE {\bfseries Output:} worst-case loss
   \STATE Split $\spec$ into $\spec_\emph{aug}$ and $\spec_\emph{abs}$
                and return
   \STATE \vspace{-5mm}\[\max_{\bfz\in \emph{augment}_k(\spec_\emph{aug},\bfx)} \mathcal{L}(\widehat{\bfz},y,\params)\ \ \ \ 
   \textit{\ s.t.\ }\ \widehat{\bfz} = \mathit{abstract}(\spec_\emph{abs},\bfz)
   \]
   \vspace{-4mm}
\end{algorithmic}
\end{algorithm}

\subsection{Computing Abstractions}\label{ssec:abstract}

We now show how to define the abstraction of a perturbation space $\spec(\bfx)$ defined by a specification
$\spec = \{(T_1,\delta_1),\allowbreak \ldots , (T_n, \delta_n)\}$.
Our approach generalizes that of \citet{DBLP:conf/emnlp/HuangSWDYGDK19} to length-preserving transformations, i.e., ones where the length of every string in
$\spec(\bfx)$ is the same as the length of the original string $\bfx$. The approach of \citet{DBLP:conf/emnlp/HuangSWDYGDK19} targeted the special case of single-character substitutions.

\textbf{Single transformation case}
We first demonstrate the case of a single length-preserving transformation, $\spec=\{(T,\delta)\}$. 
Henceforth we assume that each element of a string $\bfx$ is a real value, e.g., the embedding of a character or word. 
At a high level, our abstraction computes the convex hull that contains
all the points in $T(\bfx)$ (we use $T(\bfx)$
as a short hand for the perturbation space obtained by applying $T$ to $\bfx$
exactly once) and then scales this convex hull by $\delta$ to account
for the cases in which $T$ is applied up to $\delta$ times.
We begin by computing all points in $T(\bfx)$.
Let this set be $\bfx_0, \ldots, \bfx_m$.
Next, for $i \in [1,m]$, we define the set of points
\[
\bfv_i = \bfx + \delta\cdot(\bfx_i - \bfx).
\]
We then construct the abstraction $\mathit{abstract}(\spec_\emph{abs},\bfz)$ as the convex hull of the points $\bfv_i$ and $\bfx$.
Observe that we only need to enumerate the space $T(\bfx)$
obtained by applying $T$ once;
multiplying by $\delta$ \emph{dilates} the convex hull to include all strings that involve up to $\delta$ applications of $T$, i.e., $\spec(\bfx)$.
To propagate this convex hull through the network, we typically overapproximate
it as a set of \emph{interval} constraints, where each dimension of a string is represented by a lower and an upper bound. Interval constraints are easier to propagate through the network---requiring several forward passes linear in the length of the string---compared to arbitrary convex polyhedra, whose operations can be exponential in the number of dimensions~\cite{cousot1978automatic}.

\begin{example}\label{ex:a1}
Consider the left side of \cref{fig:abstract}.
Say we have the string \emph{ab} and the transformation 
 $T$ that can replace character $a$ with $z$, or $b$ with $g$---mimicking spelling mistakes, as these characters are adjacent on a QWERTY keyboard.
The large shaded region is the result of dilating $T$ with $\delta=2$,
i.e., contains all strings that can be produced by applying $T$ twice to the string $\emph{ab}$, namely, the string $\emph{zg}$.
\end{example}

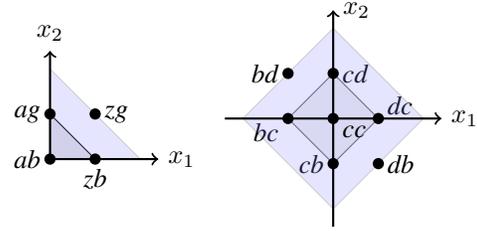
\begin{figure}
\centering
\begin{tikzpicture}[scale=1.2]

    \tkzDefPoint(0,0){A}
    \tkzLabelPoint[left](A){\emph{ab}}
    
    \tkzDefPoint(0,.5){B}
    \tkzLabelPoint[left](B){\emph{ag}}
    
    \tkzDefPoint(.5,0){C}
    \tkzLabelPoint[below](C){\emph{zb}}
    
    \tkzDefPoint(.5,.5){D}
    \tkzLabelPoint[right](D){\emph{zg}}
    
    \filldraw[draw=black, fill=gray!20] (A) -- (B) -- (C) -- cycle;
    
    \filldraw[draw=black, fill=blue!50,opacity=0.2] (0,0) -- (0,1) -- (1,0) -- cycle;

    \foreach \n in {A,B,C,D}
        \node at (\n)[circle,fill,inner sep=1.5pt]{};
        
        % Draw axes
    %the first axis is white, to balance things out
    \draw [->,thick,opacity=0] (0,0) -- (0,-.75) node (yaxis) [above] {};
    \draw [->,thick] (0,0) -- (1.2,0) node (yaxis) [right] {$x_1$};
    \draw [->,thick] (0,0) -- (0,1.2) node (yaxis) [above] {$x_2$};
\end{tikzpicture}
~
\begin{tikzpicture}[scale=1.2]

    \tkzDefPoint(0,.5){B}
    \tkzLabelPoint[right](B){\emph{cd}}
    
    \tkzDefPoint(.5,0){C}
    \tkzLabelPoint[right,yshift=.5em](C){\emph{dc}}
    
    \tkzDefPoint(-.5,0){B1}
    \tkzLabelPoint[left,yshift=-.5em](B1){\emph{bc}}
    
    \tkzDefPoint(0,-.5){C1}
    \tkzLabelPoint[left](C1){\emph{cb}}
    
    \tkzDefPoint(0.5,-.5){C2}
    \tkzLabelPoint[right](C2){\emph{db}}
    
    \tkzDefPoint(-0.5,.5){C3}
    \tkzLabelPoint[left](C3){\emph{bd}}
    
    \filldraw[draw=black, fill=gray!20, opacity=0.6] (B) -- (C) -- (C1) -- (B1) -- cycle;
    
    \filldraw[draw=black, fill=blue!50,opacity=0.2] (1,0) -- (0,1) -- (-1,0) -- (0,-1) -- cycle;
    
    \tkzDefPoint(0,0){A}
    \tkzLabelPoint[right,yshift=-.5em](A){\emph{cc}}

    \foreach \n in {A,B,C,B1,C1,C2, C3}
        \node at (\n)[circle,fill,inner sep=1.5pt]{};
        
        % Draw axes
    \draw [->,thick] (-1.2,0) -- (1.2,0) node (yaxis) [right] {$x_1$};
    \draw [->,thick] (0,-1.2) -- (0,1.2) node (yaxis) [right] {$x_2$};
\end{tikzpicture}
\caption{Illustration of an abstraction of a single (left) and multiple (right) transformations. See \cref{ex:a1} and \cref{ex:a2} for details.}
\label{fig:abstract}
\end{figure}

\textbf{General case}
We generalize the above abstraction process to the perturbation space $\spec=\{(T_1,\delta_1), \ldots, (T_n, \delta_n)\}$.
First, we enumerate all the strings 
in
$T_1(\bfx)\cup \dots \cup T_n(\bfx)$.
(Notice that we need only consider each transformation $T_i$ independently.)
Let this set be $\bfx_0, \ldots, \bfx_m$.
Next, for $i \in [1,m]$, we define the following set of points:
\begin{align*}
    \bfv_i = \bfx + (\delta_1+\ldots+\delta_n)\cdot(\bfx_i - \bfx)
\end{align*}
As with the single-transformation case, we  can now construct an abstraction of the convex hull induced by $\bfv_i$ and $\bfx$ as a set of intervals and propagate it through the network.

\begin{example}\label{ex:a2}
We illustrate the process of abstracting an input string $\emph{cc}$ on the right of Figure~\ref{fig:abstract} for the specification $\{(T_\emph{prev},1), (T_\emph{succ}, 1)\}$, where $T_\emph{prev}$ maps one character to its preceding character in the alphabet order, e.g., $c$ with $b$, and $T_\emph{succ}$ maps one character to its succeeding character in the alphabet order, e.g., $c$ with $d$.
We enumerate all the points in $T_\emph{prev}(\emph{cc})\cup T_\emph{succ}(\emph{cc})$ and compute their convex hull, shown as the inner shaded region.
This region includes all strings resulting from exclusively one application of $T_\emph{prev}$ or $T_\emph{succ}$.
Next, we dilate the convex hull by $2=1+1$ times to include all the points in the perturbation space of $\{(T_\emph{prev},1), (T_\emph{succ}, 1)\}$.
This is shown as the larger shaded region.
Notice how this region includes $\emph{bd}$ and $\emph{db}$, which result from an application of $T_\emph{succ}$ to the first character and $T_\emph{prev}$ to the second character of the original string $\emph{cc}$.
\end{example}

The following theorem states that this process is sound:
produces an overapproximation of a perturbation space $\spec(\bfx)$.
\begin{theorem}
\label{thm:soundness}
For every specification $\spec$
and input $\bfx$,
the abstracted perturbation space $\mathit{abstract}(\spec,\bfx)$ is an over-approximation of $\spec(\bfx)$---i.e.,
$
\spec(\bfx)\subseteq \mathit{abstract}(\spec,\bfx)
$
\end{theorem}
We prove Theorem~\ref{thm:soundness} in the Appendix.

% \loris{the following sentnce can go?}
% In the next section, we describe how we decompose a specification $\spec$ in our implementation of \aat.

\begin{table*}[t]
    \centering
     \caption{String transformations to construct the perturbation spaces for evaluation.}
     \vskip 0.1in
    % \setlength{\tabcolsep}{1pt}
    % \rowcolors{2}{}{gray!10}
    \begin{tabular}{clp{11cm}l}
        \toprule
        & Transformation & Description & Training\\
        \midrule
        \multirow{4}{*}{\rotatebox[origin=c]{90}{\textsc{Char}}} & \cellcolor{Gray!10}$\Tswapchar{}$ & \cellcolor{Gray!10}\textbf{swap a pair} of two adjacent characters & \cellcolor{Gray!10}Augmentation\\  
        & $\Tdelchar{}$ & \textbf{delete} a character & Augmentation \\ 
        & \cellcolor{Gray!10}$\Tinschar{}$ & \cellcolor{Gray!10}\textbf{insert} to the right of a character one of its \textbf{adjacent} characters on the keyboard & \cellcolor{Gray!10}Augmentation \\
        & $\Tsubchar{}$ & \textbf{substitute} a character with an \textbf{adjacent} character on the keyboard & Abstraction\\ 
        \midrule
        \multirow{3}{*}{\rotatebox[origin=c]{90}{\textsc{Word}}} &  \cellcolor{Gray!10}$\Tdelword{}$ &   \cellcolor{Gray!10}\textbf{delete} a \textbf{stop word} & \cellcolor{Gray!10}Augmentation \\
        & $\Tinsword{}$ &  \textbf{duplicate} a word & Augmentation \\
        & \cellcolor{Gray!10}$\Tsubword{}$ & \cellcolor{Gray!10}\textbf{substitute} a word with one of its synonyms & \cellcolor{Gray!10}Abstraction \\
        \bottomrule

    \end{tabular}
   
    \label{tab:perturbation}
\end{table*}

% \begin{enumerate}
%     \item $(T_{Swap,Char}, 1) \circ (T_{Sub,Char}, 1)$. It is a perturbation space of strings that are obtained by applying $T_{Swap,Char}$ no more than once and $T_{Sub,Char}$ no more than once.
%     \item $(T_{Del,Char}, 1) \circ (T_{Sub,Char}, 1)$. It is a perturbation space of strings that are obtained by applying $T_{Del,Char}$ no more than once and $T_{Sub,Char}$ no more than once.
%     \item $(T_{Ins,Char}, 1) \circ (T_{Sub,Char}, 1)$. It is a perturbation space of strings that are obtained by applying $T_{Ins,Char}$ no more than once and $T_{Sub,Char}$ no more than once.
%     \item $(T_{Del,Word}, 1) \circ (T_{Ins,Word}, 1) \circ (T_{Sub,Word}, 1)$. It is a perturbation space of strings that are obtained by applying $T_{Del,Word}$ no more than once, $T_{Ins,Word}$ no more than once and $T_{Sub,Word}$ no more than once.
% \end{enumerate}

\section{Experiments}
\begin{table*}[t]
    \centering
     \caption{Qualitative examples. The vanilla models correctly classify the  original samples but fail to classify the perturbed samples.}
     \vskip 0.1in
    \begin{tabular}{cl}
        \toprule
        Prediction & A character-level sample and a perturbed sample in $\{(\Tswapchar{}, 2), (\Tsubchar{},2)\}$ of AG dataset  \\
        \midrule
        Sci/Tech & ky. company wins grant to study peptides (ap) ap - a company founded by a chemistry researcher ...\\
        World & \perturbed{yk}. comp\perturbed{na}y wins gran\perturbed{f} to st\perturbed{8}dy peptides (ap) ap - a company founded by a chemistry researcher ...\\
        ~\\
        \toprule
        Prediction & A word-level sample and a perturbed sample in $\{(\Tdelword{}, 2), (\Tsubword{},2)\}$ of SST2 dataset  \\
        \midrule
        Positive & a dream cast of solid female talent who build a seamless ensemble .\\
        Negative & \perturbed{\sout{a}} \perturbed{dreaming} \perturbed{casting} of solid female talent who build \perturbed{\sout{a}} seamless ensemble .\\
        \bottomrule
    \end{tabular}
   
    \label{tab:examples}
\end{table*}
In this section, we evaluate \aat by answering
the following
research questions:
\begin{itemize}
    \item \textbf{RQ1:} Does \aat improve robustness in rich perturbation spaces for character-level and word-level models?
    \item \textbf{RQ2:} How does the complexity of the perturbation space affect the effectiveness of \aat{}?
    % \item \textbf{RQ3:} Does training on larger perturbation sizes improve robustness \{on lower perturbation sizes?}?
\end{itemize}

\subsection{Experimental Setup}
\label{sec:expsetup}
\subsubsection{Datasets and Models}
We use two datasets: 
%  The \textbf{AG} News~\cite{DBLP:conf/nips/ZhangZL15} dataset consists of a corpus of news articles collected by~\citet{DBLP:conf/www/Gulli05} about the 4 largest news topics. 
%  The Stanford Sentiment Treebank (\textbf{SST2})~\cite{DBLP:conf/emnlp/SocherPWCMNP13}  dataset  consists of sentences from movie reviews and human annotations of their sentiment. The task is to predict the sentiment (positive/negative) of a given sentence.
 \begin{itemize}
 \item \textbf{AG} News~\cite{DBLP:conf/nips/ZhangZL15} dataset consists of a corpus of news articles collected by~\citet{DBLP:conf/www/Gulli05} about the four largest news topics. We used the online-available dataset\footnote{This is the website describing the dataset: \url{https://github.com/mhjabreel/CharCnn_Keras/tree/master/data/ag_news_csv}.}, which contains 30,000 training and 1,900 testing examples for each class. We split the first 4,000 training examples for validation.
 \item \textbf{SST2}~\cite{DBLP:conf/emnlp/SocherPWCMNP13} is the Stanford Sentiment Treebank dataset that consists of sentences from movie reviews and human annotations of their sentiment. The task is to predict the sentiment (positive/negative) of a given sentence. We used the dataset provided by TensorFlow\footnote{This is the website describing the dataset: \url{https://www.tensorflow.org/datasets/catalog/glue}.}. The dataset contains 67,349 training, 872 validation, and 1,821 testing examples for each class.
\end{itemize}

\yhmodify{For the AG dataset, we trained a character-level model proposed by ~\citet{DBLP:conf/nips/ZhangZL15} following their setup.
For the SST2 dataset, we trained a word-level model proposed by~\citet{kim-2014-convolutional} also following their setup.}{For the AG dataset, we trained a smaller character-level model than the one used in ~\citet{DBLP:conf/emnlp/HuangSWDYGDK19} but kept the number of layers and the data preprocessing the same. 
For the SST2 dataset, we trained a word-level model and a character-level model. We used the same models in ~\citet{DBLP:conf/emnlp/HuangSWDYGDK19}, also following their setup.}
We provide the details of the setups in the Appendix.

\subsubsection{Perturbations}
Our choice of models allows us to experiment with both character-level and word-level perturbations.
We evaluated \aat on six perturbation spaces constructed using the seven
individual string transformations in Table~\ref{tab:perturbation}. 

For the character-level model on dataset AG, we used the following specifications:
\yhmodify{$\{(\Tswapchar{}, 1), (\Tsubchar{},1)\}$, $\{(\Tdelchar{}, 1), (\Tsubchar{},1)\}$, and $\{(\Tinschar{}, 1), (\Tsubchar{},1)\}$.}{$\{(\Tswapchar{}, 2), (\Tsubchar{},2)\}$, $\{(\Tdelchar{}, 2), (\Tsubchar{},2)\}$, and $\{(\Tinschar{}, 2), (\Tsubchar{},2)\}$.}
For example, the first specification mimics the combination of two spelling mistakes: swap two characters \yhmodify{}{up to twice} and/or substitute a character with an adjacent one on the keyboard \yhmodify{}{up to twice}.

For the word-level model on dataset SST2, we used the following specifications:
$\{(\Tdelword{},2),\allowbreak(\Tsubword{},2)\}$, $\{(\Tinsword{},2),\allowbreak(\Tsubword{},2)\}$, and $\{(\Tdelword{},2),\allowbreak(\Tinsword{},2),\allowbreak(\Tsubword{},2)\}$.
For example, the first specification removes stop words \yhmodify{(up to 2)}{up to twice} and substitutes \yhmodify{(up to 2)}{up to twice} words with synonyms. 

\yhmodify{}{
For the character-level model on dataset SST2, we used the following specifications:
$\{(\Tswapchar{}, 1), (\Tsubchar{},1)\}$, $\{(\Tdelchar{}, 1), (\Tsubchar{},1)\}$, and $\{(\Tinschar{}, 1), (\Tsubchar{},1)\}$.
For example, the first specification mimics the combination of two spelling mistakes: swap two characters and/or substitute a character with an adjacent one on the keyboard.
}

\yhmodify{For the character-level model, we considered perturbations with $\delta=1$ because one cannot efficiently evaluate the \oracle{} accuracy with larger $\delta$, due to the combinatorial explosion of the size of the perturbation space.}{For the character-level model on AG dataset, we considered the perturbations to be applied to a prefix of an input string, namely, a prefix length of 35 for $\{(\Tswapchar{}, 2), (\Tsubchar{},2)\}$, a prefix length of 30 for $\{(\Tdelchar{}, 2), (\Tsubchar{},2)\}$, and $\{(\Tinschar{}, 2), (\Tsubchar{},2)\}$.
For the character-level model on SST2 dataset, we considered perturbations with $\delta=1$ but allow the perturbations to be applied to the whole input string. We made these restrictions because one cannot efficiently evaluate the \oracle{} accuracy with larger $\delta$, due to the combinatorial explosion of the size of the perturbation space.}

\begin{table*}[t]
    \centering
    \caption{Experiment results for the three perturbations on the character-level model on AG dataset. We show the normal accuracy (Acc.), HotFlip accuracy (HF Acc.), and \oracle{} accuracy (\Oracleshort{}) of five different training methods.}
    \vskip 0.1in
    \begin{tabular}{l *3{p{\lengthcellshorter{}}p{\lengthcell{}}p{\lengthcelllonger{}}}}
        \toprule
         & \multicolumn{3}{>{\centering\arraybackslash}p{\lengththreecell{}}}{$\{(\Tswapchar{}, 2),(\Tsubchar{},2)\}$} & \multicolumn{3}{>{\centering\arraybackslash}p{\lengththreecell{}}}{$\{(\Tdelchar{}, 2),(\Tsubchar{},2)\}$} & \multicolumn{3}{>{\centering\arraybackslash}p{\lengththreecell{}}}{$\{(\Tinschar{}, 2),(\Tsubchar{},2)\}$} \\ 
         
         \cmidrule(lr){2-4} \cmidrule(lr){5-7} \cmidrule(lr){8-10}
        Training & Acc. & {\small HF Acc.} & \Oracleshort{} & Acc. & {\small HF Acc.} & \Oracleshort{} & Acc. & {\small HF Acc.} & \Oracleshort{}\\
        \midrule
        Normal & 87.5 & 71.5 & \cellcolor{Gray!10}60.1 & 87.5 & 79.0 & \cellcolor{Gray!10}62.5 & 87.5 & 79.1 & \cellcolor{Gray!10}59.0 \\ 
        Random Aug. & 87.5 & 75.7 & \cellcolor{Gray!10}68.2 [+8.1] & 87.4 & 81.3 & \cellcolor{Gray!10}69.4 [+6.9] & 87.8 & 81.2 & \cellcolor{Gray!10}69.7 [+10.7] \\
        HotFlip Aug. & 86.6 & 85.7 &\cellcolor{Gray!10} 84.9 [+24.8] & 85.8 & 84.9 & \cellcolor{Gray!10}82.7 [+20.2] & 86.8 & 85.9 & \cellcolor{Gray!10}82.6 [+23.6]\\ 
        \advab{}  & 86.4 & 86.4 & \cellcolor{Gray!10}86.4 [+26.3] & 87.2 & 87.1 & \cellcolor{Gray!10}85.7 [+23.2] & 87.4 & 87.4 & \cellcolor{Gray!10}85.5 [+26.5] \\
        \enumab{}  & 86.9 & 86.8 & \cellcolor{Gray!10}\textbf{86.8 [+26.7]} & 87.6 & 87.4 & \cellcolor{Gray!10}\textbf{86.2 [+23.7]} & 87.9 & 87.8 & \cellcolor{Gray!10}\textbf{86.5 [+27.5]}\\
        \bottomrule
    \end{tabular}
    \label{tab:table1}
\end{table*}

% 87.5	71.5	60.1	87.5	79.0	62.5	87.5	79.1	59.0
% 87.5	75.7	68.2	87.4	81.3	69.4	87.8	81.2	69.7
% 86.6	85.7	84.9	85.8	84.9	82.7	86.8	85.9	82.6
% 86.4	86.4	86.4	87.2	87.1	85.7	87.4	87.4	85.5
% 86.9	86.8	86.8	87.6	87.4	86.2	87.9	87.8	86.5

\begin{table*}[t]
    \centering
    \caption{Experiment results for the three perturbations on the word-level model on SST dataset.}
    \vskip 0.1in
    \begin{tabular}{l *3{p{\lengthcellshorter{}}p{\lengthcell{}}p{\lengthcelllonger{}}}}
        \toprule
        & \multicolumn{3}{>{\centering\arraybackslash}p{\lengththreecell{}}}{$\{\Tdelword{},2),(\Tsubword{},2)\}$} & \multicolumn{3}{>{\centering\arraybackslash}p{\lengththreecell{}}}{$\{(\Tinsword{},2),(\Tsubword{},2)\}$} & \multicolumn{3}{>{\centering\arraybackslash}p{\lengththreecell{}}}{\small $\{(\Tdelword{},2),(\Tinsword{},2),(\Tsubword{},2)\}$} \\
        \cmidrule(lr){2-4} \cmidrule(lr){5-7} \cmidrule(lr){8-10}
        Training & Acc. & {\small HF Acc.} & \Oracleshort{} & Acc. & {\small HF Acc.} & \Oracleshort{} & Acc. & {\small HF Acc.} & \Oracleshort{}\\
        \midrule
        Normal & 82.4 & 68.9 & \cellcolor{Gray!10}64.4 & 82.4 & 55.8 & \cellcolor{Gray!10}47.9 & 82.4 & 54.8 & \cellcolor{Gray!10}42.4 \\
        Random Aug.  & 80.0 & 70.0 & \cellcolor{Gray!10}66.0 [+1.6] & 81.5 & 54.2 & \cellcolor{Gray!10}49.7 [+1.8] & 81.0 & 56.1 & \cellcolor{Gray!10}46.2 [+3.8] \\
        HotFlip Aug. & 80.8 & 74.4 & \cellcolor{Gray!10}68.3 [+3.9] & 80.8 & 68.7 & \cellcolor{Gray!10}56.0 [+8.1] & 81.2 & 69.0 & \cellcolor{Gray!10}51.0 [+8.6] \\ 
        \advab{}  & 80.2 & 73.5 & \cellcolor{Gray!10}70.2 [+5.8] & 79.9 & 69.7 & \cellcolor{Gray!10}57.7 [+9.8] & 78.8 & 68.1 & \cellcolor{Gray!10}55.1 [+12.7] \\
        \enumab{} & 79.9 & 74.4 & \cellcolor{Gray!10}\textbf{71.2 [+6.8]} & 79.0 & 70.7 & \cellcolor{Gray!10}\textbf{62.7 [+14.8]} & 77.7 & 69.8 & \cellcolor{Gray!10}\textbf{59.8 [+17.4]}\\
        \bottomrule
    \end{tabular}
    \label{tab:table2}
\end{table*}

% 82.4	68.9	64.4	82.4	55.8	47.9	82.4	54.8	42.4
% 80.0	70.0	66.0	81.5	54.2	49.7	81.0	56.1	46.2
% 80.8	74.4	68.3	80.8	68.7	56.0	81.2	69.0	51.0
% 80.2	73.5	70.2	79.9	69.7	57.7	78.8	68.1	55.1
% 79.9	74.4	71.2	79.0	70.7	62.7	77.7	69.8	59.8

\begin{table*}[t]
    \centering
    \caption{Experiment results for the three perturbations on the character-level model on SST2 dataset.}
    \vskip 0.1in
    \begin{tabular}{l *3{p{\lengthcellshorter{}}p{\lengthcell{}}p{\lengthcelllonger{}}}}
        \toprule
         & \multicolumn{3}{>{\centering\arraybackslash}p{\lengththreecell{}}}{$\{(\Tswapchar{}, 1),(\Tsubchar{},1)\}$} & \multicolumn{3}{>{\centering\arraybackslash}p{\lengththreecell{}}}{$\{(\Tdelchar{}, 1),(\Tsubchar{},1)\}$} & \multicolumn{3}{>{\centering\arraybackslash}p{\lengththreecell{}}}{$\{(\Tinschar{}, 1),(\Tsubchar{},1)\}$} \\ 
         
         \cmidrule(lr){2-4} \cmidrule(lr){5-7} \cmidrule(lr){8-10}
        Training & Acc. & {\small HF Acc.} & \Oracleshort{} & Acc. & {\small HF Acc.} & \Oracleshort{} & Acc. & {\small HF Acc.} & \Oracleshort{}\\
        \midrule
        Normal & 77.0 & 36.5 & \cellcolor{Gray!10}23.0 & 77.0 & 50.7 & \cellcolor{Gray!10}25.8 & 77.0 & 51.0 & \cellcolor{Gray!10}24.4 \\ 
        Random Aug. & 75.6 & 47.1 & \cellcolor{Gray!10}28.2 [+5.2] & 75.7 & 56.4 & \cellcolor{Gray!10}29.3 [+3.5] & 74.5 & 57.0 & \cellcolor{Gray!10}33.8 [+9.4] \\
        HotFlip Aug. & 71.4 & 63.9 &\cellcolor{Gray!10} 34.8 [+11.8] & 76.6 & 67.1 & \cellcolor{Gray!10}38.0 [+12.2] & 76.1 & 70.4 & \cellcolor{Gray!10}33.4 [+9.0]\\ 
        \advab{}  & 73.6 & 54.8 & \cellcolor{Gray!10}35.2 [+12.2] & 75.3 & 58.2 & \cellcolor{Gray!10}32.9 [+7.1] & 72.4 & 66.3 & \cellcolor{Gray!10}44.7 [+20.3] \\
        \enumab{}  & 70.2 & 57.1 & \cellcolor{Gray!10}\textbf{48.7 [+15.7]} & 72.5 & 62.5 & \cellcolor{Gray!10}\textbf{44.8 [+19.0]} & 71.6 & 65.0 & \cellcolor{Gray!10}\textbf{55.2 [+30.8]}\\
        \bottomrule
    \end{tabular}
    \label{tab:table4}
\end{table*}

% 77.0	36.5	23.0	77.0	50.7	25.8	77.0	51.0	24.4
% 75.6	47.1	28.2	75.7	56.4	29.3	74.5	57.0	33.8
% 71.4	63.9	34.8	76.6	67.1	38.0	76.1	70.4	33.4
% 73.6	54.8	35.2	75.3	58.2	32.9	72.4	66.3	44.7
% 70.2	57.1	48.7	72.5	62.5	44.8	71.6	65.0	55.2

\begin{figure*}[ht]
\vskip 0.1in
\begin{center}
\subfigure[$\{(\Tdelword{}, \delta_1),(\Tsubword{}, 2)\}$]{\includegraphics[width=0.47\textwidth]{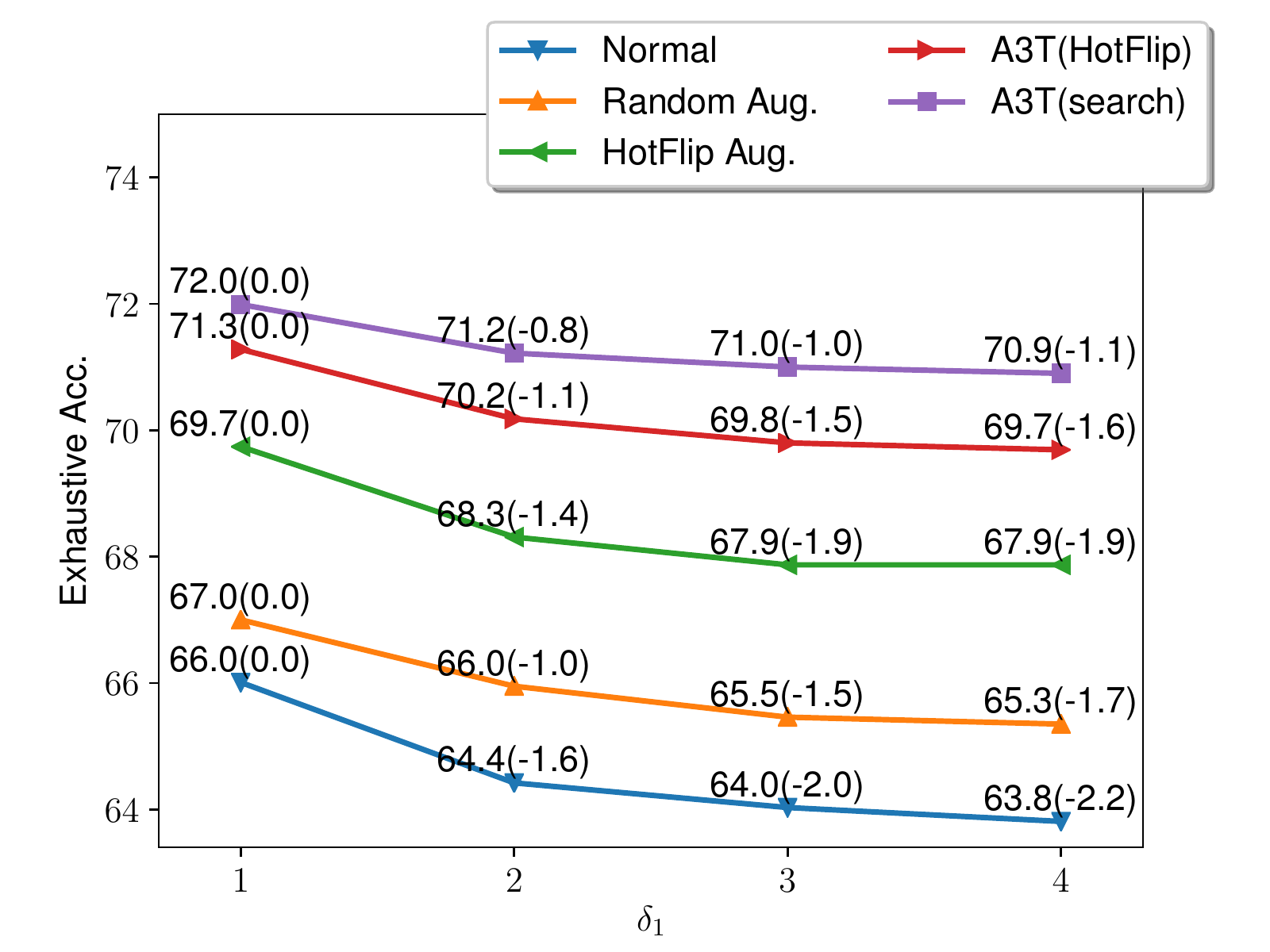}\label{fig:generalize_left}}
\subfigure[$\{(\Tdelword{}, 2),(\Tsubword{}, \delta_2)\}$]{\includegraphics[width=0.47\textwidth]{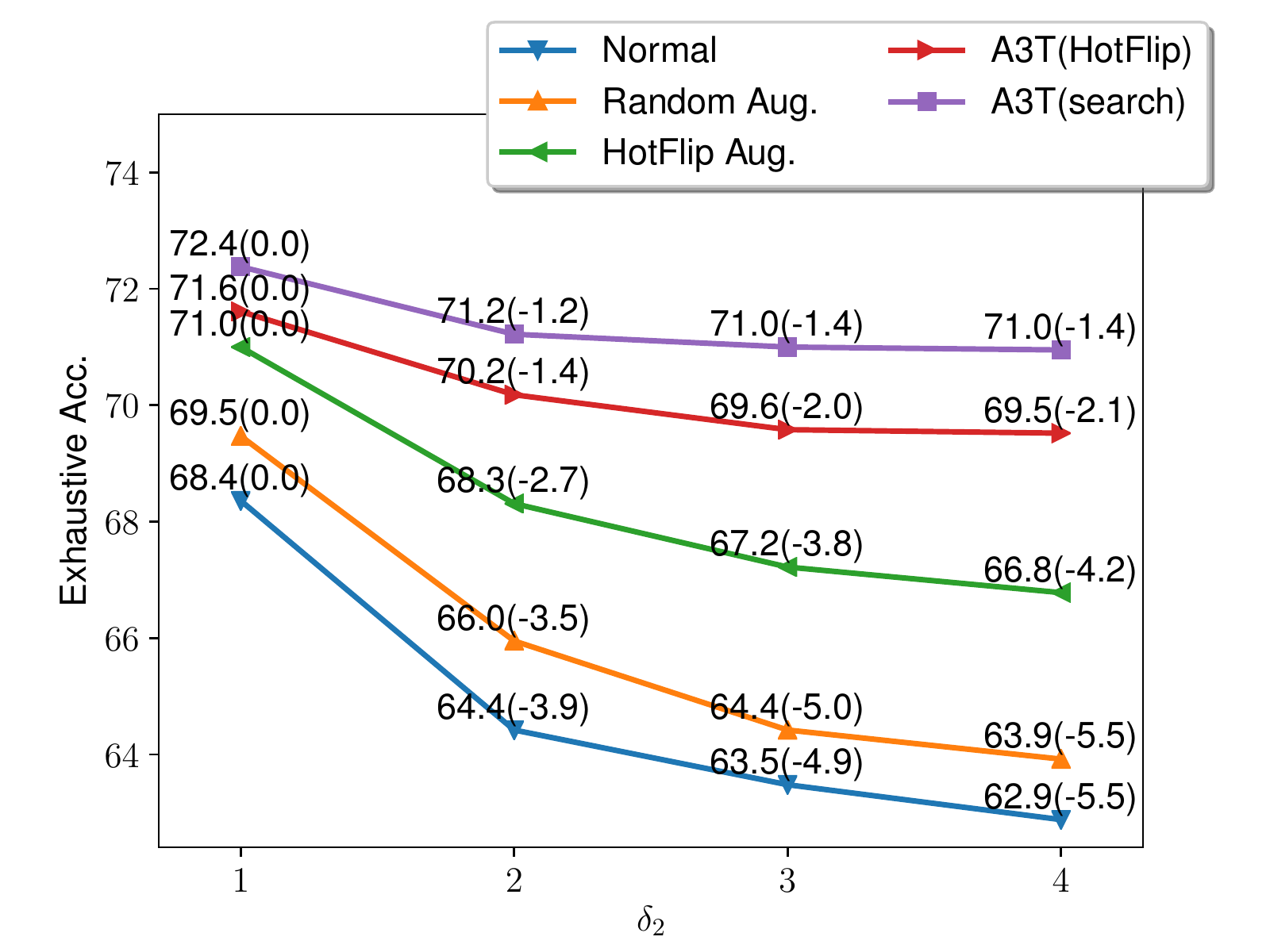}\label{fig:generalize_right}}
\caption{The \oracle{} accuracy of $\{(\Tdelword{}, \delta_1), (\Tsubword{}, \delta_2)\}$, varying the parameters $\delta_1$ (left) and $\delta_2$ (right) between 1 and 4.}
\label{fig:generalize}
\end{center}
\vskip -0.1in
\end{figure*}
\subsubsection{Training Methods}
We implement and compare the following 
training methods.

% \{I've shrunk the description. Formulas were overly complicated.}

\textbf{Normal training} is the vanilla training method (\cref{eq:normal}) that  minimizes the cross entropy between predictions and target labels.
This method does not use the perturbation space and does not attempt to train a robust model.

\textbf{Random augmentation}
performs adversarial training (\cref{eq:advtrain}) using a weak adversary that simply picks a random \perturbedsample{} from the perturbation space.
% first augments the training set with examples sampled from the perturbation space and then trains the model by the augmented set.
% For each input string $\bfx$ in the training set, we augmented the training set it with one string $\bfz$ sampled from the perturbation space.
% \{does how random sampling really matter? propose to remove}
% For a specification $\spec=\{(T_1,\delta_1), \ldots, (T_n, \delta_n)\}$,
% we produce $\bfz$ by 
% % \{the following is still mediocre}
% uniformly sampling one  string $\bfz_1$ from a string transformation $(T_1,\delta_1)$ and passing it to the next transformation $(T_2,\delta_2)$, where we then sample a new string $\bfz_2$, and so on until we have exhausted all transformations.
% This method does not guarantee uniform sampling, but it is a good enough proxy for it in practice.

\textbf{HotFlip augmentation} performs adversarial training (\cref{eq:advtrain}) using the HotFlip~\cite{DBLP:conf/acl/EbrahimiRLD18} attack to solve the inner maximization problem.

% \begin{align*}\argmin_{\params} \mathop{\mathbb{E}}_{(\bfx,y)\sim \mathcal{D}} \; (\mathcal{L}(x,y,\params) + \max_{\bfz\in\perturbe(\bfx)} \; \mathcal{L}(\bfz,y,\params))
% \end{align*}
% which is the sum of the normal loss in Eq~\ref{eq:normal} and the adversarial loss in Eq~\ref{eq:advtrain}.

\textbf{\aat{}} is our technique that can be implemented in various ways.
For our experiments, we made the following choices.
First, we manually labeled which transformations in $\spec$ are explored using augmentation and which ones are explored using abstract interpretation (the third column in Table~\ref{tab:perturbation}).\footnote{We consider this choice of split to be a hyperparameter.}
Second, we implemented two different ways of performing data augmentation for the transformations in $\spec_\emph{aug}$:
(1) \textbf{\advab{}} uses HotFlip to find the worst-case samples for augmentation, while (2) \textbf{\enumab{}} performs an explicit
search through the perturbation space to find the worst-case samples for augmentation.
Finally, we used DiffAI~\cite{diffai} to perform abstract training for the transformations in $\spec_\emph{abs}$, using the intervals abstraction.

In all augmentation training baselines, and \aat, we also adopt a curriculum-based training method~\cite{DBLP:conf/emnlp/HuangSWDYGDK19, DBLP:conf/iccv/GowalDSBQUAMK19} which uses a hyperparameter $\lambda$ to weigh between normal loss and maximization objective in \cref{eq:advtrain}.

% \begin{align*}
% \argmin_{\params} & \mathop{\mathbb{E}}_{(\bfx,y)\sim \mathcal{D}} \; ((1 - \lambda)\mathcal{L}(x,y,\params) + \\ &\lambda \max_{\bfz\in \emph{augment}_k(\spec_\emph{aug},\bfx)} \mathcal{L}(\emph{abstract}(\spec_\emph{abs}{\bfz}),y,\params)).
% \end{align*}
\subsubsection{Evaluation Metrics}
\textbf{Normal accuracy} is the vanilla accuracy of the model on the test set.

\textbf{HotFlip accuracy} is the adversarial accuracy of the model with respect to the HotFlip attack, i.e., for each point in the test set, we apply the HotFlip attack and test if the classification is still correct.

\textbf{\Oracle{} accuracy} (Eq~\ref{eq: exhaustiveacc})
is the worst-case accuracy of the model: a prediction on $(\bfx,y)$ is considered correct
if and only if all points $\bfz \in \spec(\bfx)$ lead to the correct prediction.

By definition, HotFlip accuracy is an upper bound on \oracle{} accuracy.

\subsection{Evaluation Results}
\label{sec:expresult}
\paragraph{RQ1: Increase in robustness}
We show the results for the selected perturbation spaces on character-level and word-level models in
Tables \ref{tab:table1}, \ref{tab:table2}, and \ref{tab:table4}, respectively.

Compared to normal training, the results show that both \advab{} and \enumab{} increase the \oracle{} accuracy and can improve the robustness of the model.
\advab{} and \enumab{} also
outperform random augmentation and HotFlip augmentation.
In particular, \enumab{} has \oracle{} accuracy that is
on average \yhmodify{$18.7$}{$20.3$} higher than normal training,
\yhmodify{$14.9$}{$14.6$} higher than random augmentation,
and \yhmodify{$9.6$}{$6.7$} higher than HotFlip augmentation.

\yhmodify{
We also compared \aat{} to training using only abstraction (i.e., all transformations in $\spec$
are also in $\spec_\emph{abs}$) for the specification $\{(\Tswapchar{},1),(\Tsubchar{},1)\}$ (not shown in \cref{tab:table1,tab:table2});
this is the only specification that can be fully trained abstractly since it only uses length-preserving transformations. 
Training using only abstraction yields an \oracle{} accuracy of $66.3$, which is better than the one obtained using normal training, but much lower than the \oracle{} accuracy of \advab{} and \enumab{}.
Furthermore, the normal accuracy of the abstraction technique drops to $75.3$ due to the over-approximation of the perturbation space while \advab{} ($84.6$) and \enumab{} ($87.2$) retain high normal accuracy.
}
{
We also compared \aat{} to training using only abstraction (i.e., all transformations in $\spec$
are also in $\spec_\emph{abs}$) for the specification $\{(\Tswapchar{},2),(\Tsubchar{},2)\}$ on AG dataset and $\{(\Tswapchar{},1),(\Tsubchar{},1)\}$ on SST2 dataset (not shown in Tables \ref{tab:table1}, \ref{tab:table2}, and \ref{tab:table4});
this is the only specification that can be fully trained abstractly since it only uses length-preserving transformations. 
Training using only abstraction yields an \oracle{} accuracy of $86.9$ for $\{(\Tswapchar{},2),(\Tsubchar{},2)\}$ on AG dataset, which is similar to the \oracle{} accuracy of \advab{} ($86.4$) and \enumab{} ($86.8$).
However, training using only abstraction yields an \oracle{} accuracy of $47.0$ for $\{(\Tswapchar{},1),(\Tsubchar{},1)\}$ on SST2 dataset, which is better than the one obtained using normal training, but much lower than the \oracle{} accuracy of \advab{} and \enumab{}.
Furthermore, the normal accuracy of the abstraction technique on SST2 dataset drops to $58.8$ due to the over-approximation of the perturbation space while \advab{} ($73.6$) and \enumab{} ($70.2$) retain high normal accuracy.}

% \begin{mybox}
To answer \textbf{RQ1}, \textbf{\aat{} yields models that are  more robust to complex perturbation spaces than those produced by augmentation and abstraction techniques.}
This result holds for both
character-level and word-level models.
% \end{mybox}

\paragraph{RQ2: Effects of size of the perturbation space}

In this section, we evaluate whether
\aat{} can produce models that are robust to complex perturbation spaces. 
We fix the word-level model \aat(search) trained on $\{(\Tdelword{}, 2),(\Tsubword{}, 2)\}$.
Then, we test this model's exhaustive accuracy on
$\{(\Tdelword{}, \delta_1),(\Tsubword{}, 2)\}$ (Figure~\ref{fig:generalize_left}) and
$\{(\Tdelword{}, 2),(\Tsubword{}, \delta_2)\}$
(Figure~\ref{fig:generalize_right}),
where we vary the parameters $\delta_1$ and $\delta_2$ between 1 and 4,
increasing the size of the perturbation space.
% \{I'm confused. What I wrote above matches the figure, but below you mention a case where delta1 is equal to 4 and delta2 is equal to 1, but this doesn't appear in the figure. Also, please fix figure caption with what I wrote above.}
% \{I mentioned below is a case in character level. in character level delta2 is 1}
(The Appendix contains a more detailed evaluation with different types of transformations.)
We only consider word-level models because computing the \oracle{} accuracy requires us to enumerate 
all the elements in the perturbation space.
While enumeration is feasible for word-level transformations (e.g., the perturbation space of $\{(\Tdelword{}, 4), (\Tsubword{}, 2)\}$ 
for a string with \num[group-separator={,}]{56} tokens
contains at most \num[group-separator={,}]{68002} \perturbedsample{s}),
enumeration is infeasible for character level transformations (e.g., 
\yhmodify{the perturbation space of $\{(\Tdelchar{}, 4), (\Tsubchar{}, 1)\}$ 
for a string with \num[group-separator={,}]{300} characters
contains \num[group-separator={,}]{20252321116} \perturbedsample{s}!).}{the perturbation space of $\{(\Tdelchar{}, 4), (\Tsubchar{}, 2)\}$ 
on a prefix of length 30 of string with \num[group-separator={,}]{300} characters
contains \num[group-separator={,}]{7499469} \perturbedsample{s}, and the perturbation space of $\{(\Tdelchar{}, 4), (\Tsubchar{}, 1)\}$ 
for a string with \num[group-separator={,}]{300} characters
contains \num[group-separator={,}]{20252321116} \perturbedsample{s}!).}

% $339,654,979$ examples for the character perturbation $(\Tdelchar{}, 3) \circ (\Tsubchar{}, 1)$.
% $20,252,321,116$ examples for the character perturbation $(\Tdelchar{}, 4) \circ (\Tsubchar{}, 1)$.

The \oracle{} accuracy of \advab{} and \enumab{} decreases by \yhmodify{$1.8\%$ and $2.2\%$}{$1.6\%$ and $1.1\%$}, respectively, when increasing $\delta_1$ from $1$ to $4$, and decreases by \yhmodify{$1.8\%$ and $2.5\%$}{$2.1\%$ and $1.4\%$}, respectively, when increasing $\delta_2$ from $1$ to $4$.
\yhmodify{
All other techniques result in larger decreases in \oracle{} accuracy (${\ge}1.7\%$ for all methods).}
{All other techniques result in larger decreases in \oracle{} accuracy (${\ge}1.7\%$ in $\{(\Tdelword{}, \delta_1),(\Tsubword{}, 2)\}$ and ${\ge}4.2\%$ in $\{(\Tdelword{}, 2),(\Tsubword{}, \delta_2)\}$).}

To answer \textbf{RQ2},
\textbf{even in the presence of large perturbation spaces
\aat{} yields models that are  more robust
than those produced by augmentation techniques}.

% \subsubsection{RQ3: Does increasing perturbation sizes in training improve robustness?}
% We decreased the word-level perturbation sizes of string transformations from $2$ to $1$ and trained \advab{} and \enumab{} on the decreased perturbations. 

\section{Conclusion, Limitations, and Future Work}
We presented an adversarial training technique, \aat{}, combining augmentation and abstraction techniques 
to achieve robustness against programmable string transformations
in  neural networks for NLP tasks. 
In the experiments, we showed that \aat{} yields more robust models than augmentation and abstraction techniques. 

We foresee many future improvements to \aat{}.
First, \aat{} cannot currently generalize to RNNs because its abstraction technique can only be applied to models where
the first layer is an affine transformation (e.g., linear or convolutional layer).
Applying \aat{} to RNNs will require designing new abstraction techniques for RNNs.
Second, we manually split $\spec{}$ into $\spec{}_\emph{aug}$ and $\spec{}_\emph{abs}$. Performing the split automatically is left as future work.
Third, \enumab{} achieves the best performance by
looking for the worst-case \perturbedsample{}
in the perturbation space of $\spec{}_\emph{aug}$ 
via \textit{enumeration}.
In some practical settings, $\spec{}_\emph{aug}$ 
might induce a large perturbation space, and it might be best to use \advab{} instead.
Fourth, we choose HotFlip and interval abstraction 
to approximate the worst-case loss in our experiments, 
but our approach is general and can benefit from new augmentation and abstraction techniques.

\clearpage

\section*{Acknowledgements}
We would like to thank Samuel Drews for his feedback, as well as the reviewers. This work is supported by the National Science Foundation grants CCF-1704117, CCF-1918211, and Facebook research award Probability and Programming.

\bibliography{example_paper}
\bibliographystyle{icml2020}

\clearpage
\appendix

\section{Appendix}
\subsection{Semantics of specifications}
\label{ssec:semantic}
% We define a data structure, called meta-string $\metastring{}$, for each string $s$ as a Cartesian product of $s$ and a set of positions $Pos_s$ where the transformations can perform:
% $$\metastring{}_s \triangleq s \times Pos_s$$
% , where $Pos_s \subseteq \{1,2,\dots, |s|\}$.
% Furthermore, the semantics of $\perturbe{}$ is shown as $\llbracket \perturbe{} \rrbracket = \{(\metastring{}_i, \metastring{}_o)\}$.

% We inductively define the semantics of $\perturbe{}$:
% \begin{align*}
%     \llbracket \perturbe{}_1 \circ \perturbe{}_2 \rrbracket &:= \{(\metastring{}_i, \metastring{}_o) \mid \\
%     & \exists \metastring{}_m.\ (\metastring{}_i,\metastring{}_m)\in \llbracket \perturbe{}_2 \rrbracket \wedge (\metastring{}_m,\metastring{}_o) \in \llbracket \perturbe{}_1 \rrbracket\}\\
%     \llbracket (T, 0) \rrbracket &:= \{(\metastring{}_i, \metastring{}_i)\}\\
%     \llbracket (T, \delta) \rrbracket &:= \{(\metastring{}_i, \metastring{}_i)\} \cup \{(\metastring{}_i, \metastring{}_o) \mid \\
%     & \exists \metastring{}_m.\ (\metastring{}_o,\metastring{}_m)\in \llbracket T \rrbracket \wedge (\metastring{}_m,\metastring{}_o) \in \llbracket (T, \delta - 1) \rrbracket\}\\
%     \llbracket T \rrbracket &:= \{(S_i \times Pos_i, S_o \times Pos_o) \mid \\
%     & \exists I \subseteq Pos_i.\ \varphi(i_{I_1}i_{I_2}\dots i_{I_{|I|}}) \\
%     & \wedge Pos_O = Pos_I-I  \\
%     &\wedge (o_{I_1}o_{I_2}\dots o_{I_{|I|}} = f(i_{I_1}i_{I_2}\dots i_{I_{|I|}}))\\
%     &\wedge (\forall j\not \in I.\ o_j=i_j)
% \end{align*}
We define the semantics of a specification
$\spec=\{(T_1, \delta_1), \dots, (T_n, \delta_n)\}$
(such that $T_i=(\varphi_i,f_i)$) as follows.
Given a string
$\bfx=x_1\ldots x_m$, 
a string $\bfy$ is in the perturbations space $S(\bfx)$ if:
\begin{enumerate}
    \item there
exists matches $\langle (l_1,r_1),j_1\rangle\ldots \langle (l_k,r_k),j_k\rangle$ (we assume that matches are sorted in ascending order of $l_i$)
such that for every $i\leq k$ we have that $(l_i,r_i)$ is a valid match
of $\varphi_{j_i}$ in $\bfx$;
\item the matches are not overlapping: for every two distinct $i_1$ and $i_2$, 
        $r_{i_1}<l_{i_2}$ or $r_{i_2}<l_{i_1}$;
\item the matches respect the $\delta$ constraints: for every $j'\leq n$,
      $|\{\langle (l_i,r_i),j_i\rangle \mid j_i=j'\}|\leq \delta_{j'}$.
\item the string $\bfy$ is the result of applying an appropriate transformation to each match: if for every $i\leq k$
we have $\bfs_i\in f_{j_i}(x_{l_i}\ldots x_{r_i})$, then
\[
     \bfy = x_1\ldots x_{l_1-1}\bfs_1 x_{r_1+1}\ldots x_{l_k-1}\bfs_k x_{r_k+1} \ldots x_m.
\]
\end{enumerate}

\subsection{Proof of Theorem~\ref{thm:soundness}}
We give the following definition of a convex set:
\begin{definition}{\textbf{Convex set}:}
\label{def:convex}
A set $\mathcal{C}$ is \textbf{convex} if, for all $x$ and $y$ in $\mathcal{C}$, the line segment connecting $x$ and $y$ is included in $\mathcal{C}$. 
\end{definition}

\begin{proof}
We first state and prove the following lemma.
\begin{lemma}
\label{lemma:unknown}
Given a set of points $\{p_0,p_1,\ldots,p_t\}$ and a convex set $\mathcal{C}$ such that $\{p_0,p_1,\ldots,p_t\} \subset \mathcal{C}$. These points define a set of vectors $\overrightarrow{p_0p_1},\overrightarrow{p_0p_2}, \ldots, \overrightarrow{p_0p_t}$.
If a vector $\overrightarrow{p_0p}$ can be represented as a sum weighed by $\alpha_i$:
\begin{align}
    \overrightarrow{p_0p} = \sum_{i=1}^t \alpha_i \cdot \overrightarrow{p_0p_i},
\end{align}
where $\alpha_i$ respect to constraints:
\begin{align}
    \sum_{i=1}^t \alpha_i \le 1 \wedge \forall 1\le i \le t.\; \alpha_i \ge 0,
\end{align}
then the point $p$ is also in the convex set $\mathcal{C}$.
\end{lemma} 
\begin{proof}
We prove this lemma by induction on $t$,
\begin{itemize}
    \item Base case: $t=1$, if $\overrightarrow{p_0p} = \alpha_1 \cdot \overrightarrow{p_0p_1}$ and $0 \le \alpha_1 \le 1$, then $p$ is on the segment $p_0p_1$. 
    By the definition of the convex set (Definition~\ref{def:convex}), the segment $p_0p_1$ is inside the convex, which implies $p$ is inside the convex: $p \in p_0p_1 \subseteq \mathcal{C}$.
    \item Inductive step: Suppose the lemma holds for $t=r$. If a vector $\overrightarrow{p_0p}$ can be represented as a sum weighed by $\alpha_i$:
\begin{align}
    \overrightarrow{p_0p} = \sum_{i=1}^{r+1} \alpha_i \cdot \overrightarrow{p_0p_i} \label{eq:inductivestep}
\end{align}
where $\alpha_i$ respect to constraints:
\begin{align}
    \sum_{i=1}^{r+1} \alpha_i \le 1, \label{eq:inductivecons}\\
    \forall 1\le i \le r+1.\; \alpha_i \ge 0. 
\end{align}
We divide the sum in Eq~\ref{eq:inductivestep} into two parts:
\begin{align}
    \overrightarrow{p_0p} &= \sum_{i=1}^{r+1} \alpha_i \cdot \overrightarrow{p_0p_i}\\
    &= (\sum_{i=1}^r \alpha_i \cdot \overrightarrow{p_0p_i}) + \alpha_{r+1} \cdot \overrightarrow{p_0p_{r+1}}\\
    &= (1-\alpha_{r+1})\overrightarrow{p_0p'} + \alpha_{r+1} \cdot \overrightarrow{p_0p_{r+1}} \quad \text{, and} \label{eq:no}\\ 
    \overrightarrow{p_0p'} & = \sum_{i=1}^r \frac{\alpha_i}{1-\alpha_{r+1}} \cdot \overrightarrow{p_0p_i}
\end{align}
Because from Inequality~\ref{eq:inductivecons}, we know that
\[\sum_{i=1}^r \alpha_i \le 1 - \alpha_{r+1},\]
which is equivalent to 
\[\sum_{i=1}^r \frac{\alpha_i}{1-\alpha_{r+1}} \le 1. \]
This inequality enables the inductive hypothesis, and we know point $p'$ is in the convex set $\mathcal{C}$.
From Eq~\ref{eq:no}, we know that the point $p$ is on the segment of $p'p_{r+1}$, since both two points $p'$ and $p_{r+1}$ are in the convex set $\mathcal{C}$, then the point $p$ is also inside the convex set $\mathcal{C}$.
\end{itemize}
\end{proof}

To prove Theorem~\ref{thm:soundness}, we need to show that every \perturbedsample{} $\bfy \in \spec(\bfx)$ lies inside the convex hull of $\mathit{abstract}(\spec,\bfx)$.

\textbf{We first describe the \perturbedsample{} $\bfy$.} The \perturbedsample{} $\bfy$ as a string is defined in the semantics of specification $\spec$ (see the \cref{ssec:semantic}).
In the rest of this proof, we use a function $E: \alphabet^m \mapsto \mathbb{R}^{m\times d}$ mapping from a string with length $m$ to a point in $m\times d$-dimensional space, e.g., $E(\bfy)$ represents the point of the \perturbedsample{} $\bfy$ in the embedding space.
We use $\bfx_{\langle (l,r),j,\bfs \rangle}$ to represent the string perturbed by a transformation $T_j=(\varphi_j,f_j)$ such that $(l,r)$ is a valid match of $\varphi_j$ and $\bfs \in f_j(x_l,\ldots, x_r)$. Then 
\[
     \bfx_{\langle (l,r),j,\bfs \rangle} = x_1\ldots x_{l-1}\bfs x_{r+1}\ldots x_m.
\]
We further define $\Delta_{\langle (l,r),j,\bfs \rangle}$ as the vector $E(\bfx_{\langle (l,r),j,\bfs \rangle})-E(\bfx) = \overrightarrow{E(\bfx)E(\bfx_{\langle (l,r),j,\bfs \rangle})}$:
\[
\Delta_{\langle (l,r),j,\bfs \rangle} = (\underbrace{0, \ldots, 0}_{(l-1)\times d},E(\bfs)-E(x_l\ldots x_r),\underbrace{0,\ldots, 0}_{(m-r)\times d}).
\]
A \perturbedsample{} $\bfy$ defined by matches $\langle (l_1,r_1),j_1\rangle\ldots \langle (l_k,r_k),j_k\rangle$ and for every $i\leq k$
we have $\bfs_i\in f_{j_i}(x_{l_i}\ldots x_{r_i})$, then
\[
     \bfy = x_1\ldots x_{l_1-1}\bfs_1 x_{r_1+1}\ldots x_{l_k-1}\bfs_k x_{r_k+1} \ldots x_m.
\]
The matches respect the $\delta$ constraints: for every $j' \le n$, $|\{\langle (l_i,r_i),j_i,\bfs_i\rangle \mid j_i=j'\}|\leq \delta_{j'}$.
Thus, the size of the matches $k$ also respect the $\delta$ constraints:
\begin{align}
    k = \sum_{j'=1}^n |\{\langle (l_i,r_i),j_i,\bfs_i\rangle \mid j_i=j'\}| \le \sum_{j'=1}^n \delta_{j'}. \label{eq:deltacons}
\end{align}

In the embedding space,
\begin{align*}
\overrightarrow{E(\bfx)E(\bfy)} = (\underbrace{0, \ldots, 0}_{(l_1-1)\times d},E(\bfs_1)-E(x_{l_1}\ldots x_{r_1}),\\
0,\ldots,0,E(\bfs_k)-E(x_{l_k}\ldots x_{r_k}), \underbrace{0, \ldots, 0}_{(m-r_k)\times d}).    
\end{align*}
Thus, we can represent $\overrightarrow{E(\bfx)E(\bfy)}$ using $\Delta_{\langle (l,r),j,\bfs \rangle}$:
\begin{align}
    \overrightarrow{E(\bfx)E(\bfy)} = \sum_{i=1}^k \Delta_{\langle (l_i,r_i),j_i,\bfs_i \rangle}. \label{eq:xtoy}
\end{align}

\textbf{We then describe the convex hull of $\mathit{abstract}(\spec,\bfx)$.} The convex hull of $\mathit{abstract}(\spec,\bfx)$ is constructed by a set of points $E(\bfx)$ and $E(\bfv_{\langle (l,r),i,\bfs \rangle})$, where points $E(\bfv_{\langle (l,r),i,\bfs \rangle})$ are computed by:
\[
E(\bfv_{\langle (l,r),j,\bfs \rangle}) \triangleq E(\bfx) + (\sum_{i=1}^n \delta_i) (E(\bfx_{\langle (l,r),j,\bfs \rangle})-E(\bfx)) .
\]
Alternatively, using the definition of $\Delta_{\langle (l,r),j,\bfs \rangle}$, we get
\begin{align}
    \overrightarrow{E(\bfx)E(\bfv_{\langle (l,r),j,\bfs \rangle})} = (\sum_{i=1}^n \delta_i) \Delta_{\langle (l,r),j,\bfs \rangle}. \label{eq:xtoz}
\end{align}

\textbf{We then prove the Theorem~\ref{thm:soundness}}.
To prove $E(\bfy)$ lies in the convex hull of $\mathit{abstract}(\spec,\bfx)$, we need to apply Lemma~\ref{lemma:unknown}.
Notice that a convex hull by definition is also a convex set.
Because from Eq~\ref{eq:xtoy}, we have
\begin{align*}
    \overrightarrow{E(\bfx)E(\bfy)} = \sum_{i=1}^k \Delta_{\langle (l_i,r_i),j_i,\bfs_i \rangle}\\
    = \frac{1}{\sum_{i=1}^n \delta_i} \sum_{i=1}^k (\sum_{i'=1}^n \delta_{i'}) \Delta_{\langle (l_i,r_i),j_i,\bfs_i \rangle}.
\end{align*}
We can use Eq~\ref{eq:xtoz} into the above equation, and have 
\begin{align*}
    = \frac{1}{\sum_{i=1}^n \delta_i} \sum_{i=1}^k  \overrightarrow{E(\bfx)E(\bfv_{\langle (l_i,r_i),j_i,\bfs_i \rangle})}\\
    = \sum_{i=1}^k (\frac{1}{\sum_{i=1}^n \delta_i}) \cdot \overrightarrow{E(\bfx)E(\bfv_{\langle (l_i,r_i),j_i,\bfs_i \rangle})}.
\end{align*}
To apply Lemma~\ref{lemma:unknown}, we set 
\[\alpha_i = \frac{1}{\sum_{j=1}^n \delta_{j}}. \]
Using Inequality~\ref{eq:deltacons} on
\begin{align}
    \alpha_i = \frac{1}{\sum_{j=1}^n \delta_{j}} \ge 0, \label{eq:lemmacon1}
\end{align}
we get
\begin{align}
\sum_{i=1}^k \alpha_i = \sum_{i=1}^k \frac{1}{\sum_{j=1}^n \delta_{j}} = \frac{k}{\sum_{j=1}^n \delta_{j}} \le 1. \label{eq:lemmacon2}
\end{align}
The constraints in Inequality~\ref{eq:lemmacon1} and Inequality~\ref{eq:lemmacon2} enable Lemma~\ref{lemma:unknown}, and by applying Lemma~\ref{lemma:unknown}, we know that point $E(\bfy)$ is inside the convex hull of $\mathit{abstract}(\spec,\bfx)$.
\end{proof}

\subsection{Details of Experiment Setup}
% \subsubsection{Datasets and Models}

For AG dataset, we trained a smaller character-level model than the one used in ~\citet{DBLP:conf/emnlp/HuangSWDYGDK19}.
We followed the setup of the previous work: use lower-case letters only and truncate the inputs to have at most 300 characters. 
The model consists of an embedding layer of dimension 64, a 1-D convolution layer with 64 kernels of size 10, a ReLU layer, a 1-D average pooling layer of size 10, and two fully-connected layers with ReLUs of size 64, and a linear layer. 
We randomly initialized the character embedding and updated it during training. 

For SST2 dataset, we trained the same word-level model as the one used in~\citet{DBLP:conf/emnlp/HuangSWDYGDK19}.
The model consists of an embedding layer of dimension 300, a 1-D convolution layer with 100 kernels of size 5, a ReLU layer, a 1-D average pooling layer of size 5, and a linear layer. 
We used the pre-trained Glove embedding~\cite{pennington2014glove} with dimension 300 and fixed it during training.

For SST2 dataset, we trained the same character-level model as the one used in~\citet{DBLP:conf/emnlp/HuangSWDYGDK19}.
The model consists of an embedding layer of dimension 150, a 1-D convolution layer with 100 kernels of size 5, a ReLU layer, a 1-D average pooling layer of size 5, and a linear layer. 
We randomly initialized the character embedding and updated it during training. 

For all models, we used Adam~\cite{DBLP:journals/corr/KingmaB14} with a learning rate of $0.001$ for optimization and applied early stopping policy with patience $5$.

\subsubsection{Perturbations}
We provide the details of the string transformations we used:
\begin{itemize}
    \item $\Tsubchar{}$,  $\Tinschar{}$: We allow each character substituting to one of its adjacent characters on the QWERTY keyboard.
    \item $\Tdelword{}$: We choose $\{$\emph{and}, \emph{the}, \emph{a}, \emph{to}, \emph{of}$\}$ as our stop words set. 
    \item $\Tsubword{}$: We use the synonyms provided by PPDB~\cite{pavlick-etal-2015-ppdb}. We allow each word substituting to its closest synonym when their part-of-speech tags are also matched.
\end{itemize}

\subsubsection{Baseline}
\textbf{Random augmentation}
performs adversarial training using a weak adversary that simply picks a random \perturbedsample{} from the perturbation space.
For a specification $\spec=\{(T_1,\delta_1), \ldots, (T_n, \delta_n)\}$,
we produce $\bfz$ by 
uniformly sampling one  string $\bfz_1$ from a string transformation $(T_1,\delta_1)$ and passing it to the next transformation $(T_2,\delta_2)$, where we then sample a new string $\bfz_2$, and so on until we have exhausted all transformations.
The objective function is the following:
\begin{align}\argmin_{\params} \mathop{\mathbb{E}}_{(\bfx,y)\sim \mathcal{D}} \; (\mathcal{L}(x,y,\params) + \max_{\bfz\in\perturbe(\bfx)} \; \mathcal{L}(\bfz,y,\params)) \label{eq: realtraining}
\end{align}

\textbf{HotFlip augmentation} performs adversarial training using the HotFlip~\cite{DBLP:conf/acl/EbrahimiRLD18} attack to find $\bfz$ and solve the inner maximization problem.
The objective function is the same as Eq~\ref{eq: realtraining}.

\textbf{\aat{}} adopts a curriculum-based training method~\cite{DBLP:conf/emnlp/HuangSWDYGDK19, DBLP:conf/iccv/GowalDSBQUAMK19} that uses a hyperparameter $\lambda$ to weigh between normal loss and maximization objective in \cref{eq:advtrain}.
We linearly increase the hyperparameter $\lambda$ during training. 
\begin{align*}
\argmin_{\params} & \mathop{\mathbb{E}}_{(\bfx,y)\sim \mathcal{D}} \; ((1 - \lambda)\mathcal{L}(x,y,\params) + \\ &\lambda \max_{\bfz\in \emph{augment}_k(\spec_\emph{aug},\bfx)} \mathcal{L}(\emph{abstract}(\spec_\emph{abs}{\bfz}),y,\params)).
\end{align*}
Also, we set $k$ in $\emph{augment}_k$ to $2$, which means we select $2$ \perturbedsample{s} to abstract.

\begin{figure*}[ht]
\vskip 0.1in
\begin{center}
\subfigure[$\{(\Tinsword{}, \delta_1),(\Tsubword{}, 2)\}$]{\includegraphics[width=0.49\textwidth]{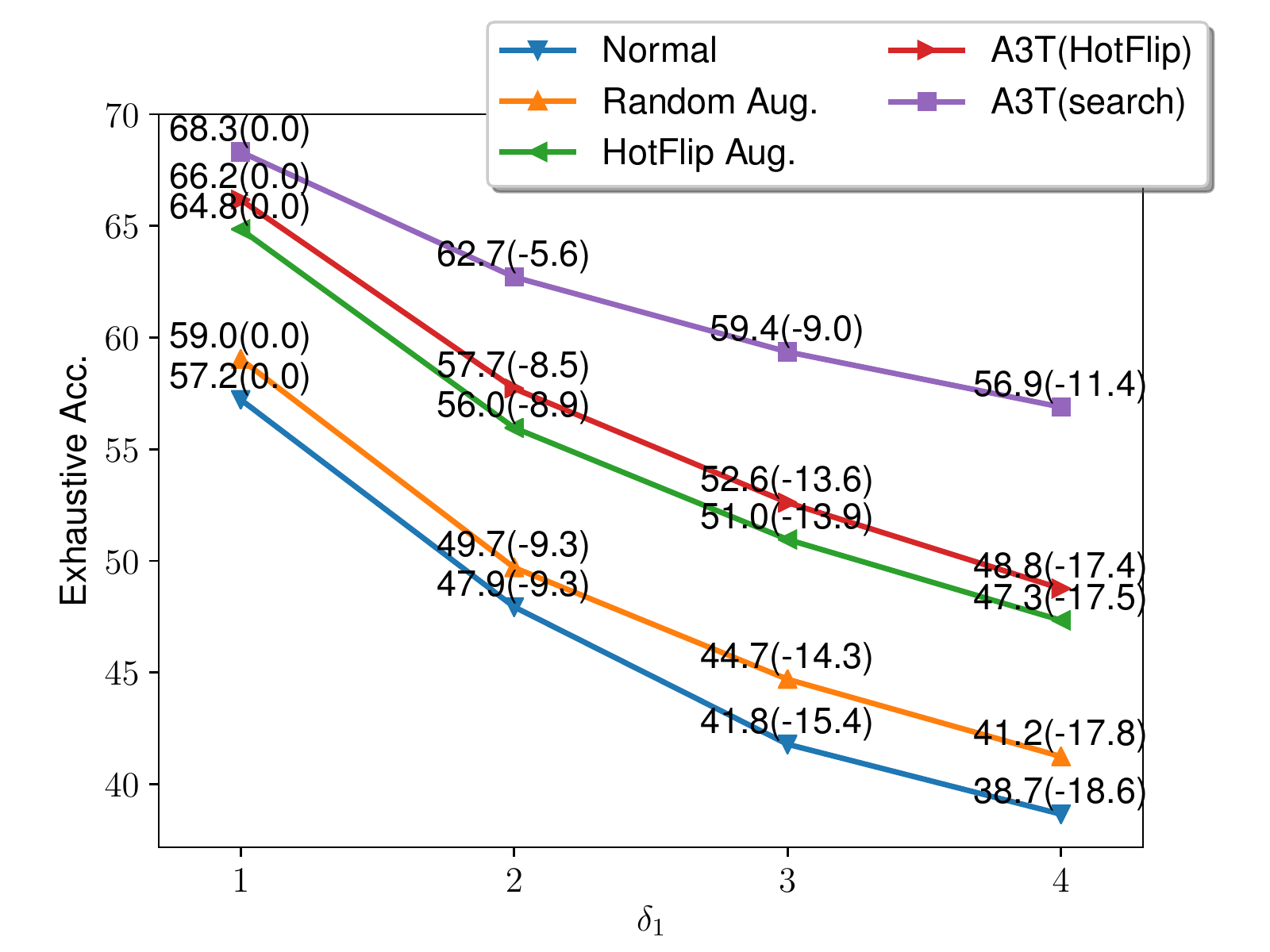}\label{fig:inssub_left}}
\subfigure[$\{(\Tinsword{}, 2),(\Tsubword{},  \delta_2)\}$]{\includegraphics[width=0.49\textwidth]{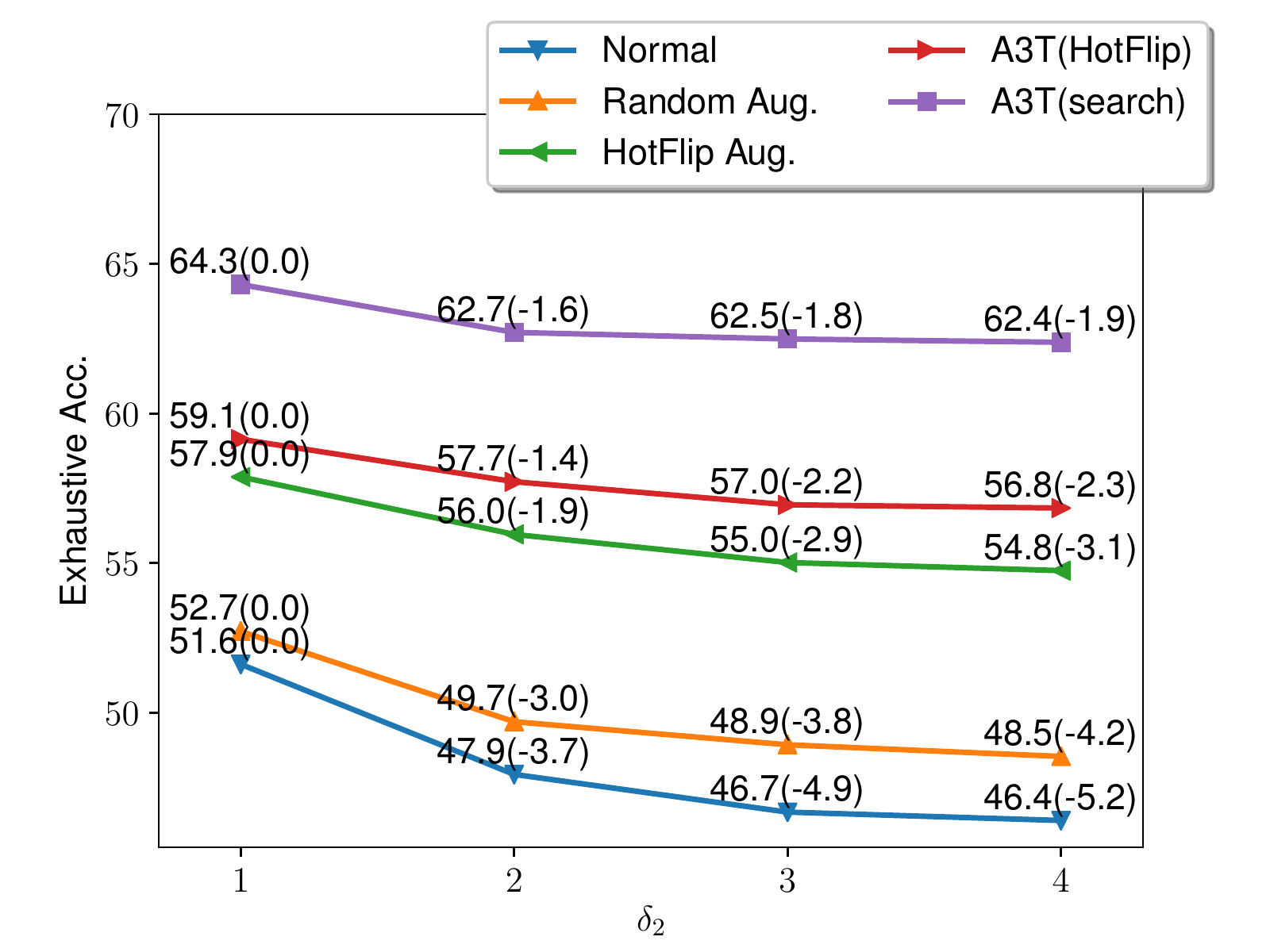}\label{fig:inssub_right}}
\caption{The \oracle{} accuracy of $\{(\Tinsword{}, \delta_1), (\Tsubword{}, \delta_2)\}$, varying the parameters $\delta_1$ (left) and $\delta_2$ (right) between 1 and 4.}
\label{fig:inssub}
\end{center}
\vskip -0.1in
\end{figure*}

\begin{figure*}[ht]
\vskip 0.1in
\begin{center}
\subfigure[$\{(\Tdelword{}, \delta_1),(\Tinsword{}, 2),(\Tsubword{}, 2)\}$]{\includegraphics[width=0.49\textwidth]{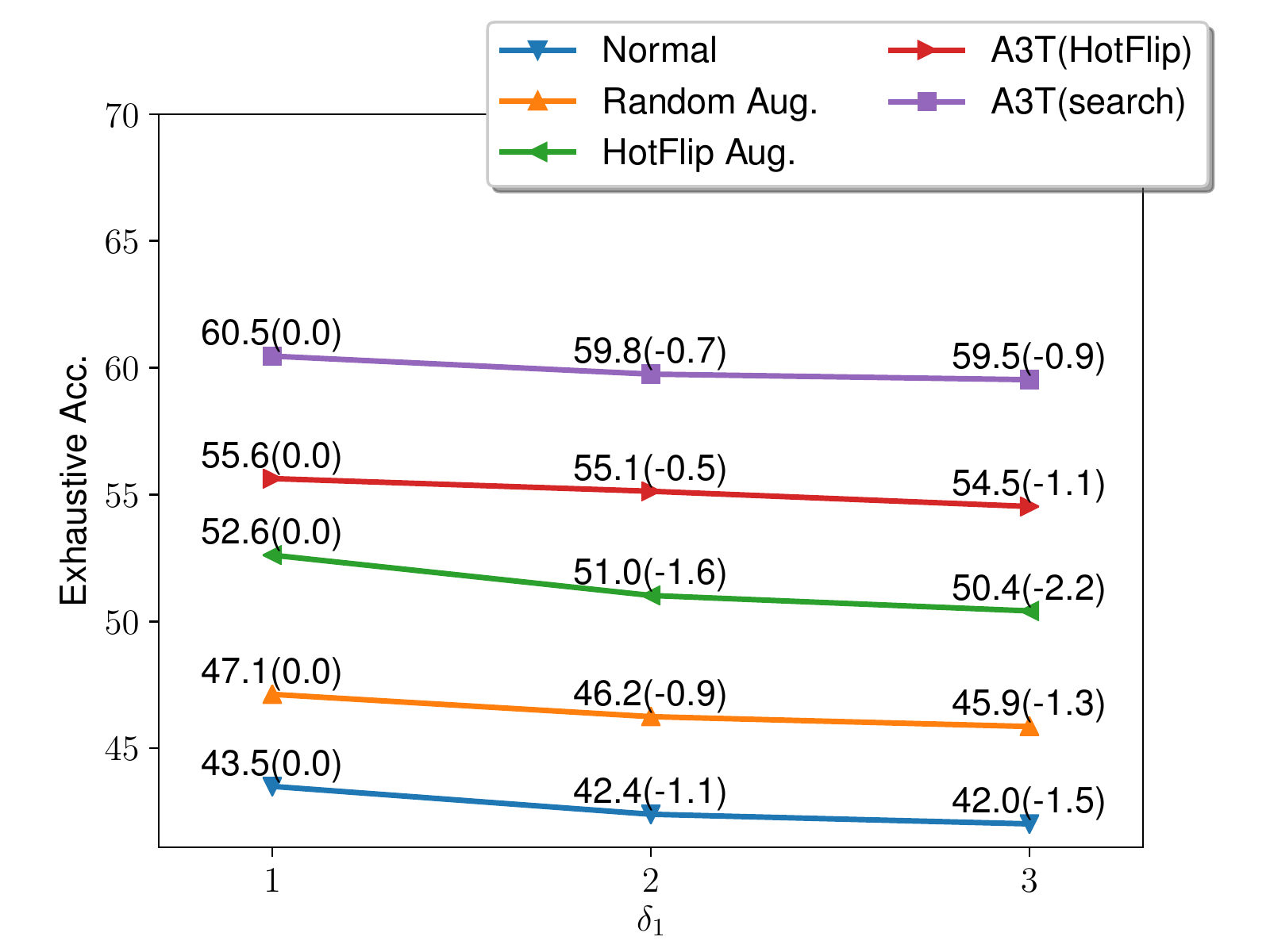}\label{fig:delinssub_left}}
\subfigure[$\{(\Tdelword{}, 2),(\Tinsword{}, \delta_2),(\Tsubword{}, 2)\}$]{\includegraphics[width=0.49\textwidth]{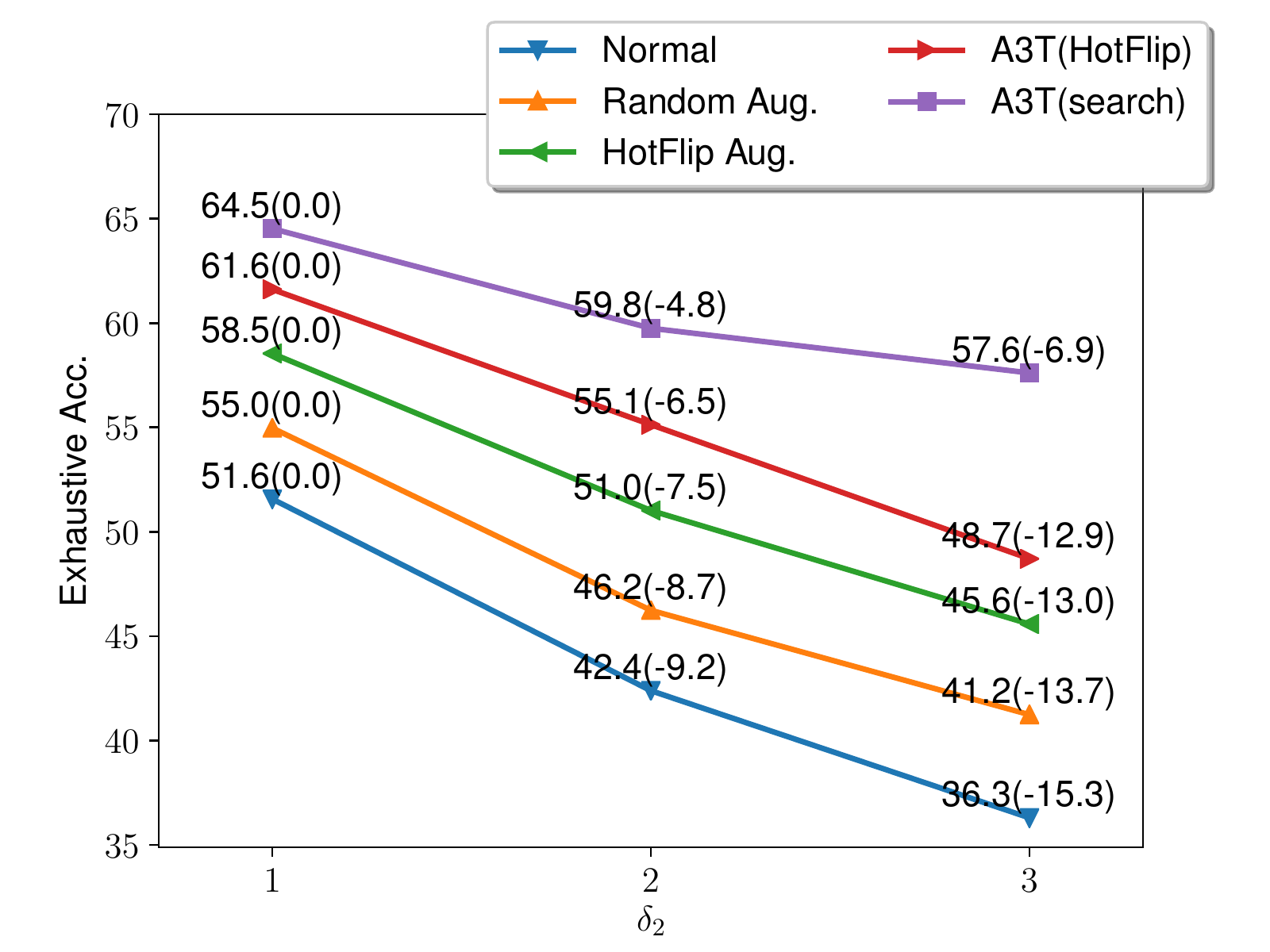}\label{fig:delinssub_mid}}
\subfigure[$\{(\Tdelword{}, 2),(\Tinsword{}, 2),(\Tsubword{}, \delta_3)\}$]{\includegraphics[width=0.49\textwidth]{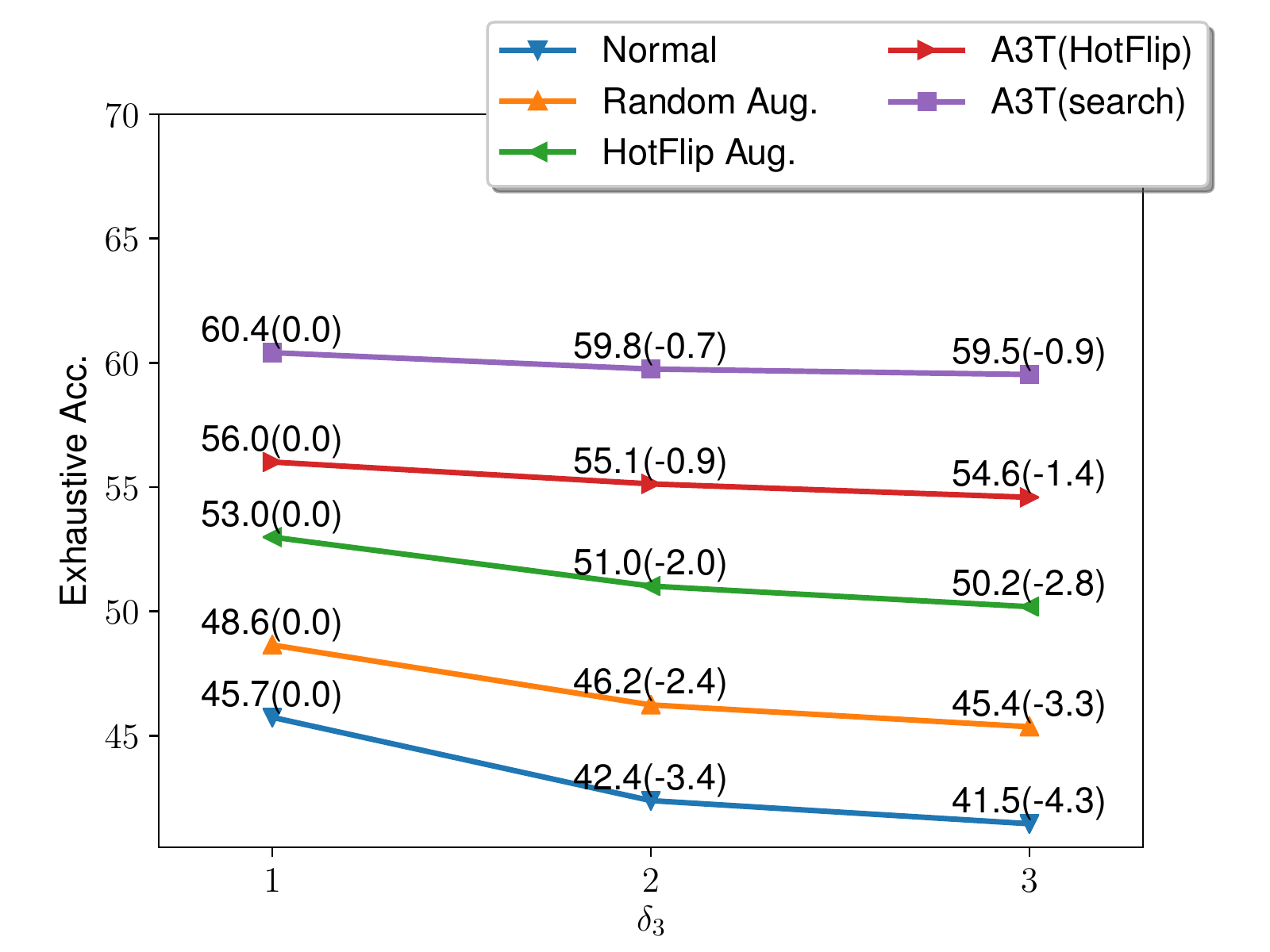}\label{fig:delinssub_right}}
\caption{The \oracle{} accuracy of $\{(\Tdelword{}, \delta_1),(\Tinsword{}, \delta_2),(\Tsubword{}, \delta_3)\}$, varying the parameters $\delta_1$ (left), $\delta_2$ (middle), and $\delta_3$ (right) between 1 and 3.}
\label{fig:delinssub}
\end{center}
\vskip -0.1in
\end{figure*}

\subsubsection{Evaluation Results}
\paragraph{RQ2: Effects of size of the perturbation space}
In Figure~\ref{fig:inssub}, we fix the word-level model \aat(search) trained on $\{(\Tinsword{}, 2),(\Tsubword{}, 2)\}$.
Then, we test this model's exhaustive accuracy on
$\{(\Tinsword{}, \delta_1),(\Tsubword{}, 2)\}$ (Figure~\ref{fig:inssub_left}) and
$\{(\Tinsword{}, 2),(\Tsubword{}, \delta_2)\}$
(Figure~\ref{fig:inssub_right}),
where we vary the parameters $\delta_1$ and $\delta_2$ between 1 and 4,
increasing the size of the perturbation space.
The \oracle{} accuracy of \advab{} and \enumab{} decreases by $17.4\%$ and $11.4\%$, respectively, when increasing $\delta_1$ from $1$ to $4$, and decreases by $2.3\%$ and $1.9\%$, respectively, when increasing $\delta_2$ from $1$ to $4$.
All other techniques result in larger decreases in \oracle{} accuracy (${\ge}17.5\%$ in $\{(\Tinsword{}, \delta_1),(\Tsubword{}, 2)\}$ and ${\ge}3.1\%$ in $\{(\Tinsword{}, 2),(\Tsubword{}, \delta_2)\}$).

In Figure~\ref{fig:delinssub}, we fix the word-level model \aat(search) trained on $\{(\Tdelword{}, 2),(\Tinsword{}, 2),(\Tsubword{}, 2)\}$.
Then, we test this model's exhaustive accuracy on
$\{(\Tdelword{}, \delta_1),(\Tinsword{}, 2),(\Tsubword{}, 2)\}$ (Figure~\ref{fig:delinssub_left}),
$\{(\Tdelword{}, 2),(\Tinsword{}, \delta_2),(\Tsubword{}, 2)\}$ (Figure~\ref{fig:delinssub_mid}),
and $\{(\Tdelword{}, 2),(\Tinsword{}, 2),(\Tsubword{}, \delta_3)\}$
(Figure~\ref{fig:delinssub_right}),
where we vary the parameters $\delta_1$, $\delta_2$ and $\delta_3$ between 1 and 3,
increasing the size of the perturbation space.
The \oracle{} accuracy of \advab{} and \enumab{} decreases by $1.1\%$ and $0.9\%$, respectively, when increasing $\delta_1$ from $1$ to $3$, decreases by $12.9\%$ and $6.9\%$, respectively, when increasing $\delta_2$ from $1$ to $3$, and decreases by $1.4\%$ and $0.9\%$, respectively, when increasing $\delta_3$ from $1$ to $3$.
All other techniques result in larger decreases in \oracle{} accuracy (${\ge}2.2\%$ in $\{(\Tdelword{}, \delta_1),(\Tinsword{}, 2),(\Tsubword{}, 2)\}$, ${\ge}13.0\%$ in $\{(\Tdelword{}, 2),(\Tinsword{}, \delta_2),(\Tsubword{}, 2)\}$, and ${\ge}2.8\%$ in $\{(\Tdelword{}, 2),(\Tinsword{}, 2),(\Tsubword{}, \delta_3)\}$).

% \paragraph{RQ3: Effects of training on different sizes of the perturbation space}

% \textbf{\emph{Do not put content after the references.}}
% %
% Put anything that you might normally include after the references in a separate
% supplementary file.

% We recommend that you build supplementary material in a separate document.
% If you must create one PDF and cut it up, please be careful to use a tool that
% doesn't alter the margins, and that doesn't aggressively rewrite the PDF file.
% pdftk usually works fine. 

% \textbf{Please do not use Apple's preview to cut off supplementary material.} In
% previous years it has altered margins, and created headaches at the camera-ready
% stage. 
%%%%%%%%%%%%%%%%%%%%%%%%%%%%%%%%%%%%%%%%%%%%%%%%%%%%%%%%%%%%%%%%%%%%%%%%%%%%%%%
%%%%%%%%%%%%%%%%%%%%%%%%%%%%%%%%%%%%%%%%%%%%%%%%%%%%%%%%%%%%%%%%%%%%%%%%%%%%%%%

\end{document}